\documentclass[10pt,twocolumn,letterpaper]{article}

\usepackage[pagenumbers]{cvpr} 

\definecolor{cvprblue}{rgb}{0.21,0.49,0.74}
\usepackage[pagebackref,breaklinks,colorlinks,allcolors=cvprblue]{hyperref}

\usepackage{algorithm}
\usepackage{algorithmic}
\usepackage{graphicx}
\usepackage{tabularx}
\usepackage{makecell}
\usepackage{xcolor}
\usepackage{bm}
\usepackage{amsmath} 
\usepackage{multirow}

\newtheorem{criterion}{Criterion}
\newtheorem{theorem}{Theorem}
\newtheorem{principle}{Principle}
\usepackage{amssymb}

\def\paperID{30566} 
\def\confName{CVPR}
\def\confYear{2026}

\title{In Search of Lost DNA Sequence Pretraining}
\author{Zhijiang Tang$^{1,2*}$, Jiaxin Qi$^{1,2*}$, Yan Cui$^{2}$, Jinli Ou$^{2}$, Yuhua Zheng$^{2}$, Jianqiang Huang$^{1,2\dagger}$ \\
$^1$Computer Network Information Center, Chinese Academy of Sciences, Beijing, China\\
$^2$Hangzhou Institute for Advanced Study, University of Chinese Academy of Sciences, Zhejiang, China\\
{\tt\small tangzhijiang24@mails.ucas.ac.cn, jxqi@cnic.cn, cuiyan.ch@gmail.com,} \\
{\tt\small oujinli@zuaa.zju.edu.cn, zhengyuhua@ucas.ac.cn, jqhuang@cnic.cn} \\
}

\begin{document}
\maketitle
\footnotetext[1]{Equal contribution. }
\footnotetext[2]{Corresponding authors.}
\begin{abstract}
DNA sequence encoding is fundamental to gene function prediction, protein synthesis, and diverse downstream biological tasks. 
Despite the substantial progress achieved by large-scale DNA sequence pretraining, 
existing studies have overwhelmingly emphasized pretraining scale and custom downstream evaluation datasets, 
while neglecting some essential components of the pretraining paradigm.
In this paper, we reveal three critical yet heretofore overlooked problems in DNA pretraining: inappropriate downstream datasets, inherent flaws in the neighbor-masking strategy, and the lack of detailed discussion on vocabulary. 
Therefore, we undertake comprehensive investigations and propose principled guidelines, including selection criteria for evaluation datasets, guiding task design, and in-depth vocabulary analysis.
Extensive experiments validate the significance of our identified problems and support the rationale behind our recommendations. Finally, we introduce a standardized testbed that enables reproducible and rigorous benchmarking of DNA pretraining methods to advance the development of genomic foundation models.

\end{abstract}    
\section{Introduction}


Encoding DNA sequences---``the language of life''---is essential for a wide range of downstream biological tasks, such as regulatory element identification~\cite{encode2012integrated} and functional variant prediction~\cite{zhou2015predicting}. 
Motivated by the success of large‐scale pretraining in natural language processing, recent works have introduced similar strategies for DNA representation learning, reporting notable gains on their downstream benchmarks~\cite{dnabert,hyenadna,evo2}.

\begin{figure}[!t]
    \centering
    \includegraphics[width=\linewidth]{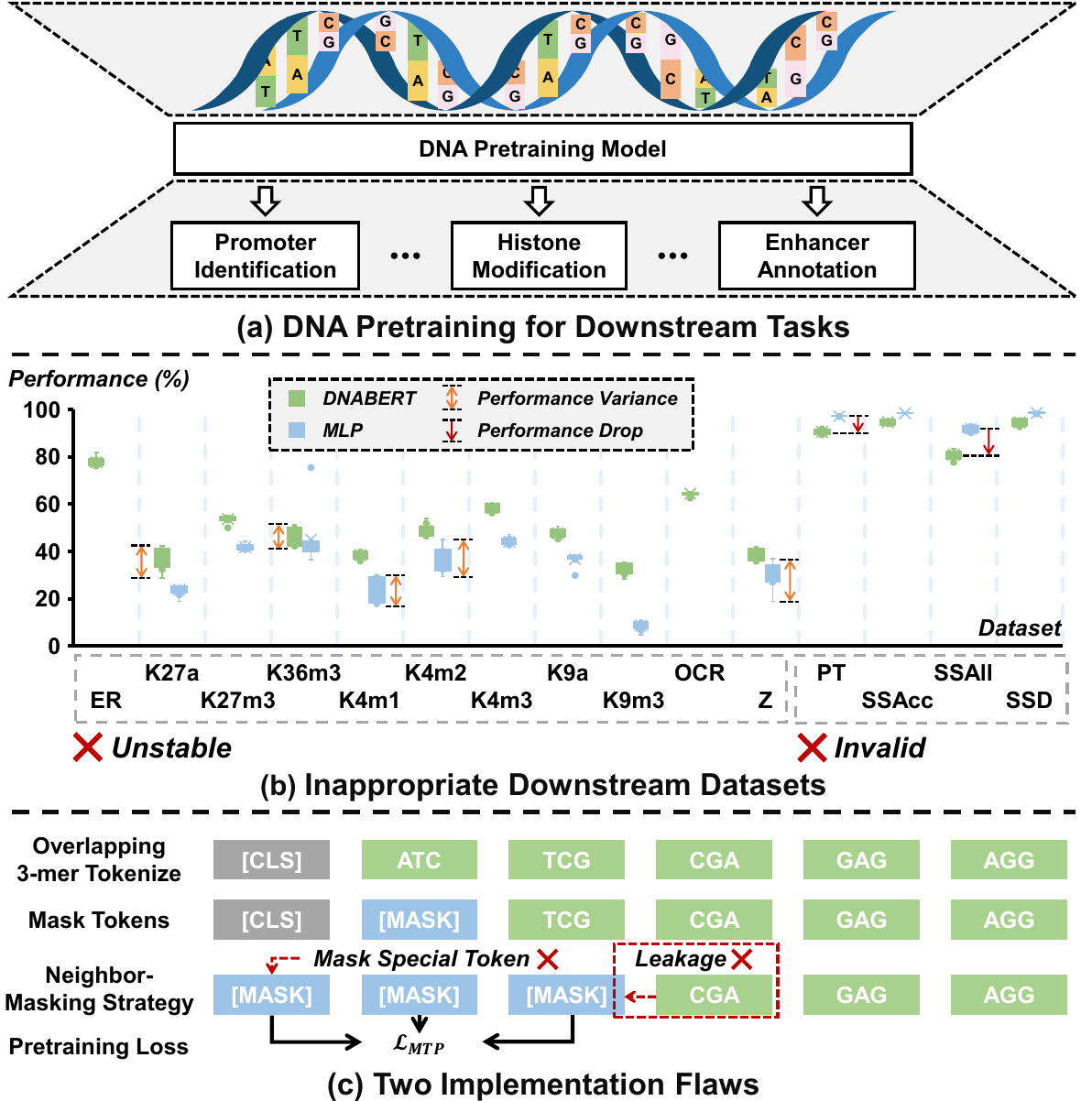}
    \caption{(a) Overview of DNA pretraining and its downstream applications.
    (b) Illustrations of model performance on inappropriate downstream datasets. (c) Two implementation flaws in neighbor-masking~\cite{dnabert}: information leakage and incorrect masking of special tokens.
    }
    \label{fig:teaser}
\end{figure}

However, existing studies have primarily focused either on scaling pretraining---by increasing data volume~\cite{evo,evo2} or expanding model capacity~\cite{NucleotideTransformer}---or on designing tailored downstream evaluation benchmarks~\cite{bend}, while overlooking fundamental principles of DNA pretraining. For example, as shown in Figure~\ref{fig:teaser},
The pretrained model underperforms a simple baseline across several custom downstream datasets,
and more concerningly, the widely adopted neighbor-masking strategy exhibits implementation flaws. 
In this paper, through in‑depth analysis and extensive experimental studies, we uncover three critical yet underexplored issues in current DNA pretraining paradigms: (a) the use of inappropriate downstream evaluation datasets, (b) implementation flaws in the neighbor‑masking strategy and the ensuing challenges when these flaws are addressed, and (c) the absence of comprehensive discussions on DNA vocabulary construction. Building on these insights, we propose concrete recommendations to mitigate these issues and establish better practices for robust DNA pretraining.


\textit{(a) Inappropriate Downstream Datasets}. Unlike pretraining in natural language processing, where generated text provides humans with an intuitive measure of model quality, DNA sequences consist of obscure nucleotides (i.e., A, T, C, G), and consequently, the effectiveness of DNA pretraining models can be evaluated only through downstream tasks.
Surprisingly, existing methods largely overlook this critical context and often resort to indiscriminate use of biologist‑provided datasets without careful selection. For example, as shown in Figure~\ref{fig:teaser}(b), we find that some downstream benchmarks exhibit high instability: identical experiments with different random seeds produce substantial variance, likely due to noisy or poorly curated data. Furthermore, in some downstream datasets, large‑scale pretrained models even fail to outperform simple end‑to‑end baselines, and some datasets contradict the established pretraining scaling law (i.e., increased pretraining scale should improve performance), suggesting potentially overly simplistic or biased data in these benchmarks. Such phenomena severely undermine the validity of model evaluations and raise significant concerns regarding the claimed advantages of current pretraining methods. 

To address this issue, we introduce two objective criteria for selecting downstream benchmarks: (1) stability, which ensures consistent model performance across multiple runs; and (2) validity, which requires that pretrained models outperform simple downstream baselines and adhere to the scaling law. Additionally, we formalize clear filtering guidelines based on these criteria and conduct extensive experiments, retaining only those datasets that reliably and meaningfully evaluate DNA pretraining methods.


\textit{(b) Implementation Flaws in Neighbor-Masking}.
Since DNA is composed of only four nucleotides, it is typically tokenized using the $k$-mer strategy~\cite{sievers2017k}. For example, ``ATC'' is treated as a single token in the 3-mer DNA tokenizer. 
To prevent word-like splitting from disrupting the inherent meaning of the DNA sequence, the $k$-mer strategy is often implemented together with overlapping windows~\cite{BarcodeBERT}. 
However, this leads to information leakage between neighbor tokens when applying the masked token prediction strategy during training. As shown in Figure~\ref{fig:teaser}(c), when encoding ``ATCG'' into two overlapping tokens ``ATC'' and ``TCG'', masking ``ATC'' as prediction target and providing ``TCG'', the overlapping ``TC'' will leak to the model. To address this problem, a neighbor-masking strategy has been proposed~\cite{dnabert}, in which both the token and its neighbors are masked.
However, we identify two critical flaws in its current implementation, illustrated in Figure~\ref{fig:teaser}(c): (1) incomplete neighbor-masking, where masked tokens at the edges still suffer from leakage when used as prediction targets; and (2) incorrect masking of special tokens, where code errors result in special tokens, such as the [CLS] token, being mistakenly masked, which is meaningless for learning in DNA pretraining.

Interestingly, we find that these flaws coincidentally facilitate DNA pretraining, as correcting them leads to model convergence failure under the same experimental settings. This may be because the leakage and meaningless special token prediction reduce the training complexity. Therefore, we introduce the guiding task strategy to reveal the nature of such phenomena. Further experiments are conducted to validate the rationale behind our findings, demonstrating that our proposed strategy realizes better model optimization. 

\textit{(c) Vocabulary Construction}. Current approaches to constructing DNA vocabularies directly adopt methods from natural language processing, such as n-grams~\cite{manning1999foundations} and byte pair encoding~(BPE)~\cite{bpe}. However, there are inherent differences between DNA sequences and language. In natural language, vocabularies are built around semantics, whether by merging multiple words in n-grams or identifying sub-words in BPE. In contrast, DNA sequences are composed of nucleotides, which can be viewed as characters in NLP, but do not inherently carry semantic meaning. Moreover, the semantic units in DNA sequences remain poorly defined and understudied, necessitating more in-depth exploration. Thus, we conducted a systematic analysis of existing DNA vocabularies based on our proposed testbed, performing a fine-grained investigation into token frequencies and learning difficulties. From this, we distilled several principles, which we hope will guide future research on DNA semantic vocabulary construction.

Our main contributions can be summarized as follows:
\begin{enumerate}
    \item Through comprehensive analysis, we identify three critical yet overlooked issues in current DNA pretraining: (a) inappropriate downstream datasets, (b) implementation flaws in neighbor-masking, and (c) insufficient discussion on vocabulary construction.
    \item We propose corresponding solutions: (a) two criteria for dataset selection, (b) the guiding task strategy, and (c) principles for DNA vocabulary construction.    
    \item We conduct extensive comparative experiments, and the results validate the rationale behind our findings and the effectiveness of our recommendations. Our proposed reproducible and robust testbed provides a solid foundation for advancing DNA pretraining research.

\end{enumerate}

\noindent\textbf{Code: }\href{https://github.com/ZhijiangTang/DNA-Tricks}{https://github.com/ZhijiangTang/DNA-Tricks}


\section{Related Works}
\noindent\textbf{DNA Downstream Evaluation.}
DNA downstream evaluation encompasses a broad spectrum of biological tasks. Functional region prediction seeks to identify significant genomic regions, such as CpG islands~\cite{gardiner1987cpg}, and datasets are typically derived from initiatives like FANTOM5~\cite{FANTOM5} and GENCODE~\cite{GENCODE}.
Histone modification identification~\cite{kouzarides2007chromatin} relies on databases such as JASPAR~\cite{JASPAR} and ChIP-Atlas~\cite{ChIPAtlas}.
Scattered downstream tasks are also consolidated into benchmarks. The BEND benchmark~\cite{bend} defines several representative tasks. The Nucleotide Transformer benchmark~\cite{NucleotideTransformer} is designed to evaluate the Nucleotide Transformer, covering tasks such as regulatory annotation~\cite{encode2012integrated}. The Genomic Benchmark~\cite{Genomicbenchmarks} provides a series of datasets for classification.
However, these widely used benchmarks generally lack in-depth analyses of their effectiveness for evaluating pretraining models. In this paper, we systematically study existing datasets and propose two criteria for selecting appropriate downstream benchmarks for DNA pretraining evaluation.

\noindent\textbf{DNA Pretraining Models.}
With the great success of pretraining models in NLP, more researchers have begun pretraining DNA models to improve performance on downstream tasks. DNAGPT~\cite{zhang2023dnagpt} uses the non-overlapping $k$-mer tokenizer and introduces a GPT-style pretraining paradigm~\cite{radford2018improving}. Nucleotide Transformer~\cite{NucleotideTransformer} also adopts the non-overlapping $k$-mer tokenizer and identifies the optimal model from among BERT-based architectures~\cite{devlin2019bert}.
DNABERT~\cite{dnabert}, BarcodeBERT~\cite{BarcodeBERT}, and GENERator~\cite{GENERator} introduce the BERT architecture into DNA sequence modeling and use overlapping $k$-mer. DNABERT-2~\cite{dnabert2} reconstructs the vocabulary using BPE. Evo~\cite{evo} builds a large-scale DNA pretraining model by hybridizing HyenaDNA~\cite{hyenadna} and Roformer~\cite{su2024roformer}. Evo2~\cite{evo2} further extends the architecture using multi-scale hyena operators~\cite{ku2025systems} and expands the pretraining dataset.
These methods focus on scaling up pretraining or architectural improvements, which leads them to overlook critical yet essential issues in DNA pretraining. In this paper, we identify these issues and provide corresponding recommendations, with extensive experiments to demonstrate the effectiveness of our solutions.


\section{Method}

\subsection{DNA Sequence Pretraining}

The basic element of DNA, nucleotide, is denoted as \(c \!\in\! \mathcal{S}_c\!=\! \{A, T, G, C\}\), and a DNA sequence can be represented as the ordered nucleotide series \( \bm{x} \!=\! (c_1, c_2, \dots, c_n) \). Instead of encoding individual nucleotides, current methods rely on custom vocabularies $\mathcal{V}$ to tokenize DNA sequences.
For example, the $k$-mer vocabulary~\cite{dnabert} is defined as $\mathcal{V}_{kmer}\!=\!\{c_1c_2\cdots c_k \,|\, \forall i \!\in\!\{1,\dots,k\},c_i\!\in\!\mathcal{S}_c\}\cup \mathcal{S}_{sp}$, where $k$ is the number of nucleotides in a token and $\mathcal{S}_{sp}$ is the set of special tokens such as $[\text{CLS}]$ and $[\text{MASK}]$. 

Given token vocabulary $\mathcal{V}$, DNA sequence $\bm{x}$ can be further represented as \( \bm{x} \!=\! (t_1, t_2, \dots, t_m), t_i\!\in\!\mathcal{V}\). Following the masked token prediction (MTP) pipeline in natural language pretraining~\cite{devlin2019bert},
given the dataset $\mathcal{D}\!=\!\{\bm{x}\}_{i=1}^N$, 
the input is masked $\bm{x}$, where some tokens are randomly masked, denoted as \({\bm{\tilde{x}}} = (t_{1}, [\text{MASK}], \ldots, t_{m})\), 
the set of masked tokens for prediction in $\bm{x}$ is denoted as $\mathcal{M}$. Then, the training objective for $\bm{x}$ can be written as:
\begin{equation}
\begin{aligned}
\label{eq:base_objective}
    &\mathcal{L}_{\text{MTP}}(\bm{\tilde{x}},\mathcal{M},\mathcal{V})\\
    &= -\frac{1}{|\mathcal{M}|} \sum\limits_{t_i\in \mathcal{M}} \bm{y}_{i} \log ( \frac{\exp(\bm{e}_{i} \bm{w})}{\sum_{k} \exp((\bm{e}_{i} \bm{w})_k)} ),
\end{aligned}
\end{equation}
where $\bm{y}_{i}$ is the one-hot class label for $i$-th masked token in sample $\bm{x}$, $\bm{e}_{i} \in \mathbb{R}^{1 \times d}$ is the features of tokens $t_{i}$, $d$ is the hidden dimension, $\bm{w} \in \mathbb{R}^{d \times |V|}$ is a linear projection mapping hidden features to vocabulary logits, and $k\in\{1,\dots,|\mathcal{V}|\}$ indexes the vocabulary dimension. 
The total training objective for dataset $\mathcal{D}$ is defined as the summation of the loss in Eq.~\eqref{eq:base_objective}, omitted here for brevity.

\begin{algorithm}[tb]
\caption{Fixed neighbor-masking strategy}
\label{alg:dnabert_masking}
\textbf{Input}: Input $\bm{x}$, $k$-mer, masking probability $\textit{p}$\\
\textbf{Output}: Masked input $\bm{\tilde{x}}$ and labels $\mathcal{M}$\\
\vspace{-10pt}
\begin{algorithmic}[1]
\STATE \textcolor{red}{\textbf{[BUG]}} Initialize $Nbr=\{1-\left \lfloor \frac{k}{2} \right \rfloor ,\cdots,k-\left \lfloor \frac{k}{2} \right \rfloor\}_{k}$
\STATE \textcolor{green}{\textbf{[FIX]}} Initialize $Nbr=\{-k+1,\cdots,k-1\}_{2k-1}$ 
\STATE Randomly select tokens from $\bm{x}$ into $\mathcal{M}$ with $p$
\STATE Remove special tokens in $\mathcal{M}$
\STATE Set $\mathcal{M}_{\text{in}}=\mathcal{M}$, $\bm{\tilde{x}}={\bm{x}}$
\FOR{each token $t_i$ in $\mathcal{M}$}
    \FOR{each offset $j$ in $Nbr$}
        \IF{$0\le i+j<\operatorname{Length}({\bm{x}})$}
            \STATE $\mathcal{M}_{\text{in}}=\mathcal{M}_{\text{in}}\cup\{\tilde{\bm{x}}\,[i+j]\}$
            \STATE $\bm{\tilde{x}}\,[i+j]=\text{[MASK]}$
        \ENDIF
    \ENDFOR
\ENDFOR
\STATE \textcolor{red}{\textbf{[BUG]}} Set labels  $\mathcal{M}=\mathcal{M}_{\text{in}}$
\STATE \textcolor{green}{\textbf{[FIX]}} Set labels $\mathcal{M}=\mathcal{M}$
\STATE \textbf{return} Masked input tokens $\bm{\tilde{x}}$, labels $\mathcal{M}$
\end{algorithmic}
\end{algorithm}

\subsection{Two Downstream Dataset Selection Criteria}



As noted above, downstream performance is the only metric for evaluating DNA models, and unreliable benchmarks can lead to misleading conclusions and catastrophic resource waste. However, we observe that several current downstream benchmarks exhibit substantial variability in model performance across repeated runs. Therefore, we define the stability criterion, which selects the downstream dataset that consistently reflects a model’s true capabilities.


\begin{criterion}[Stability]
\label{cri:stability}
Let $\sigma_\mathcal{D}$ denote the standard deviation of the model performance on the downstream dataset \( \mathcal{D}\). Assume \( \log \sigma_{\mathcal{D}} \sim \mathcal{N}(\mu, \sigma^2) \).
We say that dataset \( \mathcal{D} \) satisfies the criterion of \emph{stability} if:
$$
\log \sigma_{\mathcal{D}} < \mu + \sigma.
$$
\end{criterion}



Moreover, we observe that on specific downstream benchmarks, the pretrained model underperforms a simple linear baseline trained from scratch, and that increasing the volume of training data can even degrade performance. To mitigate such unexpected phenomena, we define the validity criterion, which selects the downstream dataset that demonstrates the benefits of large‐scale pretraining.

\begin{criterion}[Validity]
\label{cri:validity}
A dataset \( \mathcal{D} \) is considered \emph{valid} if it satisfies the following two conditions:
\begin{enumerate}
    \item {Pretraining Benefit:} On dataset \( \mathcal{D} \), the performance of a pretrained model should be better than that of a linear model trained from scratch.  
    \item {Scaling Law:} The performance on dataset \( \mathcal{D} \) improves monotonically (or at least non-decreasingly) with increasing amounts of pretraining data.
\end{enumerate}
\end{criterion}


By applying our two selection criteria to widely used downstream datasets, we derive a standardized testbed that enables rigorous and reproducible benchmarking of DNA pretraining methods. This testbed underpins all subsequent analyses presented in this paper.


\begin{figure}
    \centering
    \includegraphics[width=1\linewidth]{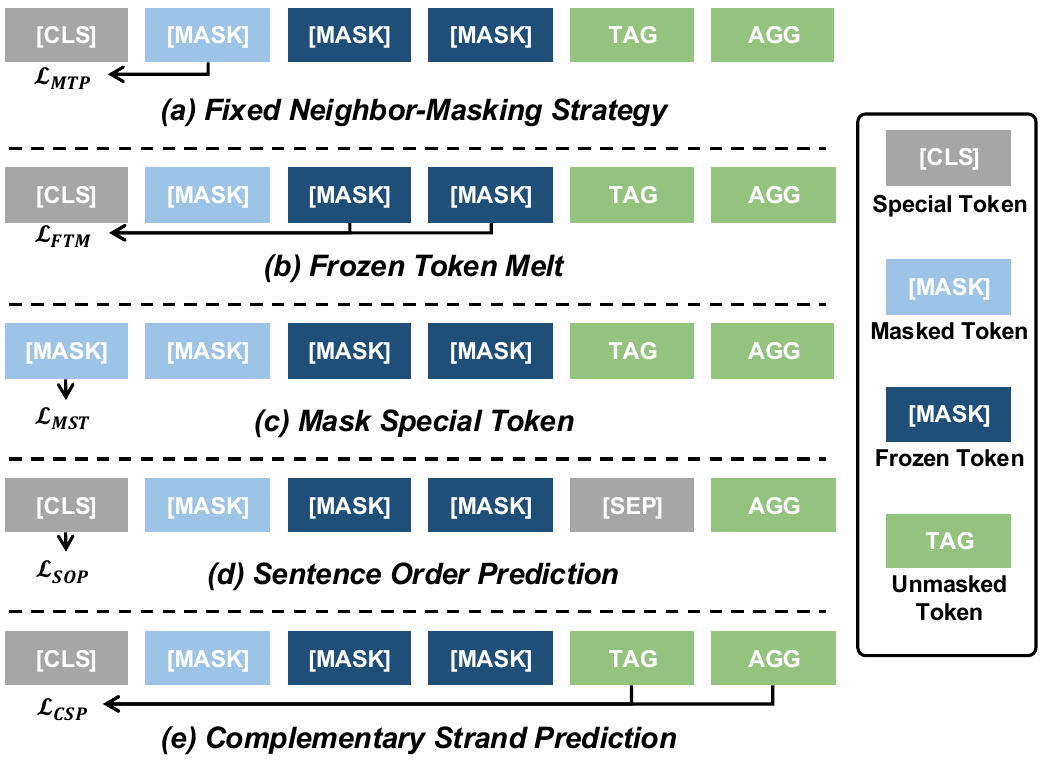}
    \caption{Illustrations of our guiding tasks. Frozen tokens are masked at input but excluded from the prediction targets to prevent information leakage. 
    }
    \label{fig:guiding_task}
    \vspace{-4pt}
\end{figure}

\subsection{Two Implementation Flaws and Guiding Tasks}
Overlapping $k$‑mer tokenization inherently introduces information leakage into the masked token prediction (as discussed above). The neighbor‑masking strategy is proposed to mitigate this by masking the target token $t_i$ together with its neighbors, such as $t_{i+1}$. However, since these neighbors remain included as masked targets in Eq.~\eqref{eq:base_objective}, which is considered an implementation flaw, leakage persists. 

Another implementation flaw of neighbor‑masking is that special tokens are inadvertently masked~\cite{dnabert}. In particular, when the first meaningful token is masked as the target, the preceding special token $[\text{CLS}]$ is mistakenly masked and included in the loss calculation, as shown in Figure~\ref{fig:teaser}(c), which produces meaningless training signals.
As shown in Algorithm~\ref{alg:dnabert_masking}, we fix the above two flaws and summarize the following principle for preventing leakage.
\begin{principle}
\label{pri:leakage}
In the overlapping $k$-mer tokenization implementation, each masked token contains partial information of its $k\!-\!1$ preceding and $k\!-\!1$ succeeding neighbors, which should be masked but excluded from the prediction targets.
\end{principle}



According to the principle, for the prediction targets, i.e., the masked tokens set $\mathcal{M}$, we could define the set of input masked token indices as:
\begin{gather}
    \mathcal{M}_{in} \!=\! \left\{ t_i \, |\ t_i \!\in\! \mathcal{T}_x, t_i \!\notin\! \mathcal{S}_{sp}, \exists\, t_j \!\in\! \mathcal{M}, |i - j| < k \right\},
\end{gather}
where $t_j$ is the masked token, and $\mathcal{S}_{sp}$ is the special token set, $\mathcal{T}_x$ is the set of all tokens in $\bm{x}$. 

Then we can correct the implementation flaws by updating the input sequence $\bm{\tilde{x}}$ 
using $\mathcal{M}_{\text{in}}$ to derive new masked input $\bm{\tilde{x}}'$: if $t_i \in \mathcal{M}_{\text{in}}$, then replace $t_i$ with $\text{[MASK]}$. Consequently, Eq.~\eqref{eq:base_objective} can be rewritten as:
\begin{gather}
\label{obj_ours}
\mathcal{L} = \mathcal{L}_{\text{MTP}}(\bm{\tilde{x}}',\mathcal{M},\mathcal{V}).
\end{gather}





Surprisingly, when we correct the implementation flaws with Eq.\eqref{obj_ours}, the model fails to converge under the same experimental conditions. We hypothesize that the original flaws inadvertently provided a form of optimization guidance, like lowering the learning rate or using a more suitable initialization. To capture this effect explicitly, we introduce the term \textbf{Guiding Task} and reformulate Eq.~\eqref{obj_ours} as:
\begin{gather}
\mathcal{L} = \mathcal{L}_{\text{MTP}}(\bm{\tilde{x}}',\mathcal{M},\mathcal{V}) + \mathcal{L}_{\text{G}},
\end{gather}
where \(\mathcal{L}_{\mathrm{G}}\) is the guiding task loss, which was coincidentally and implicitly realized by the implementation flaws. To validate this hypothesis, we explicitly formalize the two flaws and design novel guiding tasks that instantiate $\mathcal{L}_{\mathrm{G}}$.

(1) \textit{Frozen Tokens Melt} (FTM).
We formalize the implementation flaw of neighbor-masking leakage:
\begin{gather}
    \mathcal{L}_{\text{G}}=\mathcal{L}_{\text{MTP}}(\bm{\tilde{x}}',\mathcal{M}_{in} \!\setminus\! \mathcal{M},\mathcal{V}),
\end{gather}
where $\mathcal{M}_{in} \!\setminus\! \mathcal{M}$
is the set of target tokens containing leakage.

(2) \textit{Masking Special Token} (MST).
We formalize the implementation flaw of wrongly masking the special tokens:
\begin{gather}
    \mathcal{L}_{\text{G}}=\mathcal{L}_{\text{MTP}}(\bm{\tilde{x}}',\mathcal{S}_{sp},\mathcal{V}).
\end{gather}

(3) \textit{Sentence Order Prediction} (SOP).
We incorporate the SOP task~\cite{devlin2019bert} for guiding DNA pretraining:
\begin{gather}
    \mathcal{L}_{G}={y} \log \left( \sigma(\bm{e}_0\bm{w}) \right) + (1 \!-\! {y}) \log \left( 1\!-\!\sigma(\bm{e}_0\bm{w}) \right),
\end{gather}
where $\bm{y}$ is the binary label, $\sigma$ is the sigmoid function, $\bm{e}_0$ denotes the embedding of the special token [CLS].

\begin{table*}[!t]
  \centering
    \renewcommand{\arraystretch}{1.2}
    \scalebox{0.83}{
\begin{tabularx}{1.2\textwidth}{l|l|>{\centering\arraybackslash}X>{\centering\arraybackslash}X>{\centering\arraybackslash}X>{\centering\arraybackslash}X>{\centering\arraybackslash}X>{\centering\arraybackslash}X>{\centering\arraybackslash}X>{\centering\arraybackslash}X>{\centering\arraybackslash}X>{\centering\arraybackslash}X>{\centering\arraybackslash}X>{\centering\arraybackslash}X>{\centering\arraybackslash}X}
    \toprule
    \multicolumn{2}{c|}{\multirow{1}[0]{*}{Datasets}} & NET   & NE    & EE    & CA    & OCR   & K36m3 & ER    & Z     & K4m3  & K4m1  & K27m3 & K9a   & K4m2  \\
        \hline
    \multicolumn{1}{l|}{\multirow{2}[0]{*}{Evaluation}}&$\mu_{\mathcal{D}}$  & 71.87 & 82.9  & 72.27 & 71.48 & 64.14 & 42.04 & 77.41 & 40.53 & 59.61 & 38.79 & 53.19 & 48.61 & 50.71 \\
          &$\sigma_{\mathcal{D}}$& 0.07  & 0.08  & 0.29  & 0.3   & 0.43  & 0.53  & 1.26  & 1.28  & 1.37  & \textbf{\textcolor{red}{1.45 }} & \textbf{\textcolor{red}{2.85}}  & \textbf{\textcolor{red}{2.93}}  & \textbf{\textcolor{red}{3.81}}  \\
          \hline
    \multirow{2}[0]{*}{Pretrain} &$\mu_{\mathcal{D}}$& 82.97 & 71.53 & 72.34 & 71.55 & 64.37 & 53.68 & 77.71 & 39.12 & 58.31 & 38.5  & 47.09 & 47.7  & 48.76\\
         &$\sigma_{\mathcal{D}}$ & 0.07  & 0.26  & 0.51  & 0.1   & 0.34  & 0.35  & 0.22  & \textbf{\textcolor{red}{2.44}}  & \textbf{\textcolor{red}{1.5}}   & 1.23  & \textbf{\textcolor{red}{3.57}}  & 1.08  & \textbf{\textcolor{red}{1.66}}  \\
    \bottomrule
\end{tabularx}%
    }
  \caption{Test performance~(\%) (Accuracy, AUROC or MCC) on some downstream datasets. 
  Datasets with red numbers exhibit high performance variance, suggesting they are inappropriate for evaluating pretraining models according to Criterion~\ref{cri:stability}.
  }
  \label{tab:stability}%
\end{table*}


\begin{table*}[!t]
  \centering
    \renewcommand{\arraystretch}{1.2}
    \scalebox{0.83}{
\begin{tabularx}{1.2\textwidth}{l|>{\centering\arraybackslash}X>{\centering\arraybackslash}X>{\centering\arraybackslash}X>{\centering\arraybackslash}X>{\centering\arraybackslash}X>{\centering\arraybackslash}X>{\centering\arraybackslash}X>{\centering\arraybackslash}X>{\centering\arraybackslash}X>{\centering\arraybackslash}X>{\centering\arraybackslash}X>{\centering\arraybackslash}X>{\centering\arraybackslash}X}
    \toprule
    Datasets    & CpG   & HM    & EC    & NTP   & K27a  & K9m3  & K20m1 & PA    & PNT   & PT    & SSAcc & SSAll & SSD \\
    \hline
    CNN   & 80.19  &\underline{ 71.22 } & \underline{68.70} & 22.57  & 46.70  & 91.81  & 92.81  & 89.42  & 72.13  & 48.19  & 76.31   & \textbf{\textcolor{red} {91.29}} & 32.74   \\
    Unet  & \underline{85.30}  & 70.23  & 60.94  & \underline{25.90}  & 45.97  & \underline{93.19 } &\underline{ 93.70}  & \textbf{\textcolor{red}{97.56} } & 94.18  & 78.90  & \underline{96.00} & 54.60  & \textbf{\textcolor{red}{35.10} }  \\
    MLP   & 77.17  & 67.96  & 66.68   & 9.72  & \underline{51.86}  & 93.05  & 93.34  & \underline{97.48}  & \textbf{\textcolor{red}{98.62} } & \textbf{\textcolor{red}{93.35} } & \textbf{\textcolor{red}{98.69} } & 82.95  & 25.30 \\
    DNABERT & \textbf{89.67 } & \textbf{76.83 } & \textbf{71.69 } & \textbf{31.93 } & \textbf{55.16 } & \textbf{94.33 } & \textbf{94.98 } & 91.75  & \underline{95.51}  &\underline{ 81.84}  & 94.74  & \underline{83.39}  & \underline{32.83}  \\
    \bottomrule
\end{tabularx}%
    }
  \caption{Test performance~(\%) (Accuracy, AUROC or MCC) on some downstream datasets. Bold/red numbers indicate the best performance, underlined numbers indicate the runner-up. Datasets with red numbers indicate that the simple baseline surpasses the large-scale pretraining models, suggesting they are inappropriate for evaluating pretraining methods. 
  }
  \label{tab:validity}%
  \vspace{12pt}
\end{table*}

(4) \textit{Complementary Strand Prediction} (CSP).
To prevent perturbing the embedding of masked tokens, we propose CSP as the guiding task, which requires the model to predict the reverse complement of unmasked tokens:
\begin{gather}
\mathcal{L}_{\text{G}}=\mathcal{L}_{\text{MTP}}(\bm{\tilde{x}}',\mathcal{T}_x\!\setminus\!\mathcal{M},\bar{\mathcal{V}}),
\end{gather}
where $\bar{\mathcal{V}}\!=\!\left\{ \operatorname{RC}(t)\,\middle|\,t\in\mathcal{V}\right\}$, and $\operatorname{RC}$ computes the reverse complement of $t$, i.e., $A\!\to\! T$, $T\!\to\! A$, $C\!\to\! G$, $G\!\to\! C$.

Figure~\ref{fig:guiding_task} illustrates our guiding tasks. Note that, although our guiding tasks are derived from $k$-mer tokenization, they are applicable (except FTM) to any pretraining framework that employs masked token prediction.

\subsection{Vocabulary Construction}
Besides $k$-mer approach, the Word tokenizer---essentially a non-overlapping $k$-mer tokenizer---is widely used~\cite{NucleotideTransformer}, and thus shares the same vocabulary as $k$-mer. Moreover, Byte Pair Encoding (BPE) \cite{bpe} has recently been applied to DNA pretraining\cite{dnabert2}, which is defined as follows:
\begin{gather}
    \mathcal{V}_{\text{BPE}}=\mathcal{S}_c\cup\{ a_i\circ b_i\,|\,i=1,2,\cdots,n\}\cup \mathcal{S}_{sp},
\end{gather}
where $a_i \circ b_i$ denotes a new token generated by the BPE merge rule, and $n$ is the number of BPE merge operations.

Based on our extensive experiments and in-depth analyses, we propose two principles for DNA vocabularies.

\begin{principle}[Tokenization Selection]
\label{pri:tokenization}
    Employ tokenization strategies that preserve sequence continuity, e.g., BPE instead of Word, and select an appropriate vocabulary size to maintain local context and potential motif structures without disrupting biological information.
\end{principle}

\begin{principle}[Token Selection]
\label{pri:token}
Not all tokens generated by NLP-based methods are beneficial, and the following two guidelines help retain only the effective ones:
\begin{enumerate}
    \item Retain high-frequency and information-rich tokens, as they often correspond to biologically meaningful fragments and make important positive contributions.
    \item Remove low-frequency or high-entropy tokens, which often represent noise or overfitting to specific fragments, to reduce interference and improve generalization.
\end{enumerate}
\end{principle}

\begin{figure}
    \centering
    \includegraphics[width=1\linewidth]{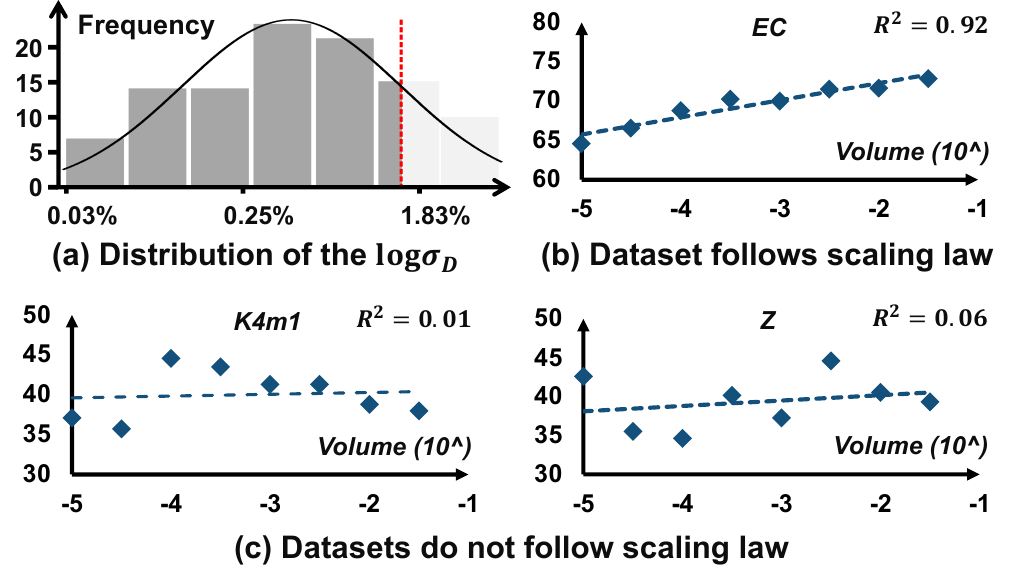}
    \caption{Illustrations of our downstream dataset selection criteria. (a) Distribution of $\log\sigma_{\mathcal{D}}$, where $\sigma_{\mathcal{D}}$ denotes the standard deviation of DNABERT. The distribution follows a normal distribution (i.e., Shapiro-Wilk test~\cite{shaphiro1965analysis} p-value is 0.15). We excluded datasets with $\sigma_{\mathcal{D}} > \exp(\mu+\sigma) = 1.41\%$ based on Criterion~\ref{cri:stability}. (b) A downstream dataset follows the scaling law. (c) Two downstream datasets that do not follow the scaling law. More details are provided in the supplementary material.}
    \label{fig:criteria}
    \vspace{-16pt}
\end{figure}

\begin{table*}[!t]
  \centering
    \renewcommand{\arraystretch}{1.2}
    \scalebox{1}{
\begin{tabularx}{1\textwidth}{l|>{\centering\arraybackslash}X>{\centering\arraybackslash}X>{\centering\arraybackslash}X>{\centering\arraybackslash}X>{\centering\arraybackslash}X>{\centering\arraybackslash}X>{\centering\arraybackslash}X>{\centering\arraybackslash}X}
    \toprule
    Model & CA    & CpG   & HM    & EC    & NE    & PA    & PNT   & Average \\
    \midrule
    HyenaDNA & \underline{71.13} & \underline{84.59} & 75.07 & 68.24 & \underline{80.68} & \underline{93.62} & 94.1  & 81.06 \\
    DNABERT2 & 68.8  & 85.42 & \underline{75.3}  & \textbf{72.78} & 79.97 & 93.49 & \underline{94.27 }& \underline{81.43 }\\
    DNABERT & \textbf{71.28} & \textbf{89.6} & \textbf{76.67} & \underline{71.45} & \textbf{82.88} & \textbf{94.25} & \textbf{95.14} & \textbf{83.04} \\
    \bottomrule
\end{tabularx}%
    }
  \caption{Test performance~(\%) (Accuracy, AUROC or MCC) of SOTA DNA pretraining model on selected downstream datasets. Bold numbers indicate the best performance, underlined numbers indicate the runner-up. 
  }
  \label{tab:performance}%
\end{table*}

\underline

\section{Experiments}
\subsection{Downstream Datasets}
The following introduces the downstream datasets that are widely used for evaluating DNA pretraining models.  

\noindent\textbf{BEND Benchmark}~\cite{bend} introduces multiple downstream DNA tasks and collects the corresponding datasets. We select three supervised tasks: 
\begin{itemize}
    \item Chromatin accessibility prediction~(CA) needs to predict chromatin accessibility across multiple cell types. 
    \item Histone modification~(HM) prediction requires predicting whether specific histone modifications exist. 
    \item CpG methylation prediction~(CpG) requires predicting whether specific sites are methylated in different cells. 
\end{itemize}
We use Area Under the Receiver Operating Characteristic Curve (AUROC) ~\cite{fawcett2006introduction} as the evaluation metric.

\noindent\textbf{Genomic Benchmark}~\cite{Genomicbenchmarks} includes multiple datasets, which can be categorized into two tasks:
\begin{itemize}
    \item Regulatory annotation aims to identify and annotate regulatory elements in DNA sequences, such as promoters and enhancers. The dataset includes Enhancers Cohn~(EC), Enhancers Ensembl~(EE), Ensembl Regulatory~(ER), and NonTata Promoters~(NTP).  
    \item Chromatin accessibility prediction: its dataset, namely OCR Ensembl~(OCR).
\end{itemize}
Following previous works~\cite{hyenadna,MxDNA}, we use accuracy as the evaluation metric.

\noindent\textbf{Nucleotide Transformer Benchmark.}~\cite{NucleotideTransformer} includes three types of tasks:  
\begin{itemize}
    \item Histone modification prediction: This benchmark includes 10 datasets, namely H2AFZ (Z), H3K27ac (K27a), H3K27me3 (K27m3), H3K36me3 (K36m3), H3K4me1 (K4m1), H3K4me2 (K4m2), H3K4me3 (K4m3), H3K9ac (K9a), H3K9me3 (K9m3), and H4K20me1 (K20m1), using Matthews Correlation Coefficient~(MCC)~\cite{gorodkin2004comparing} as the evaluation metric.
    
    \item Regulatory annotation: This benchmark includes five datasets, namely Enhancers (NE), Enhancers Types (NET), Promoter All (PA), Promoter NoTata (PNT), and Promoter Tata (PT). Except for the NET evaluation metric, which uses accuracy, all other tasks use AUROC.
    \item Splice site annotation needs to identify sites at exon-intron boundaries. Datasets include Splice Sites Acceptor~(SSAcc), Splice Sites All~(SSAll), and Splice Sites Donor~(SSD). Except for the SSAll evaluation metric, which uses accuracy, all other tasks use AUROC.

\end{itemize}

More detailed information about each dataset can be found in the supplementary materials.

\subsection{Implementation Details}
{(a) Criterion.} We obtain the estimates of the mean $\mu$ and variance $\sigma^2$ of $\log(\sigma_{\mathcal{D}})$ through extensive experiments. We evaluate the performance of three baselines: Multilayer Perceptrons (MLP)~\cite{MLP}, Convolutional Neural Networks (CNN)~\cite{bend}, and Unet~\cite{Unet}, and compare them with DNABERT to verify pretraining benefit. We model the scaling law using linear regression~\cite{hastie2009elements}. 
{(b) Guiding Task.}
We conduct experiments based on the fixed version of DNABERT with different guiding tasks, and include the Kaiming initialization~\cite{he2015delving} method for comparison. The probability of a reverse in SOP is set to 0.01.
{(c) Vocabulary Construction.}
All BPE vocabularies are trained on the human genome~\cite{GRCh38p14}. In the token ablation experiments, the removed tokens are replaced with the new $\text{[CUL]}$, noting that $\text{[CUL]}$ is not a special token. And the number of removed tokens does not exceed 10\% of the original.

We pretrain and test state-of-the-art DNA pretraining models. DNABERT~\cite{dnabert}, HyenaDNA~\cite{hyenadna}, and DNABERT2~\cite{dnabert2} use a hidden dimension of $d=256$ with a total of 8 Transformer layers. 
All pretrainings use human genome data and were conducted on NVIDIA A800 GPUs with a batch size of 64 over 100 epochs. We use the Adam optimizer~\cite{kingma2014adam} for optimization with a learning rate of $5 \times 10^{-4}$. The masking probability is 0.11.
All test experiments were conducted on NVIDIA A800 GPUs with 100 epochs of three different seeds, using the Adam optimizer. Following prior work~\cite{bend}, we freeze the pretraining model's backbone and only fine-tune a CNN with a learning rate of $5 \times 10^{-4}$. Other models undergo full training with multiple learning rates $\{1\!\times\!10^{-3}, 5\!\times\!10^{-4}, 1\!\times\!10^{-4}\}$. Different datasets use different batch sizes. More details are provided in the supplementary material.

\begin{figure}
    \centering
    \includegraphics[width=1\linewidth]{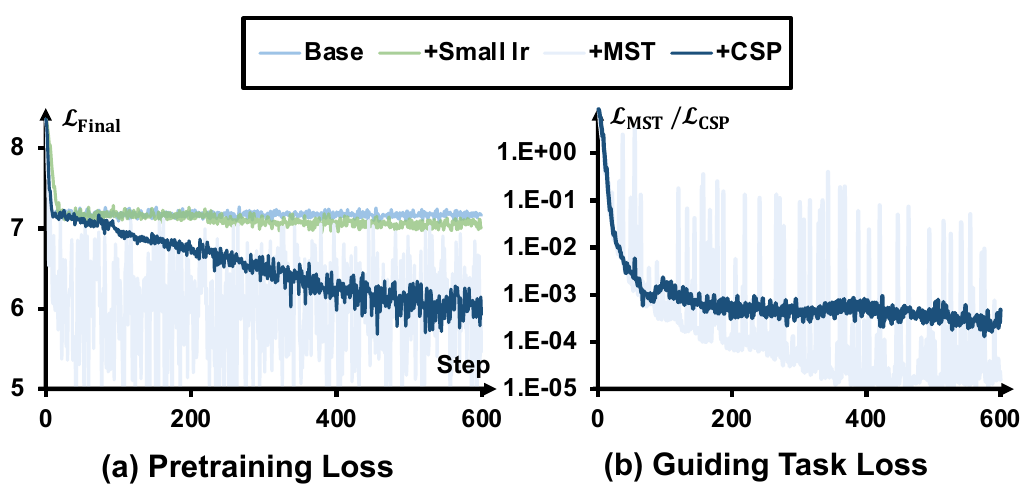}
    \caption{
    Visualization of Pretraining loss curves for different guiding tasks and the comparison of two guiding task losses; ``Base'' denotes DNABERT with a BPE tokenizer.
    }
    \label{fig:guiding_task_loss}
\end{figure}

\begin{table}[!t]
  \centering
    \renewcommand{\arraystretch}{1.2}
    \scalebox{1}{
\begin{tabularx}{0.46\textwidth}{l|>{\centering\arraybackslash}X>{\centering\arraybackslash}X>{\centering\arraybackslash}X>{\centering\arraybackslash}X}
    \toprule
    Guiding Task& $k$-mer & Word  & BPE   & Average \\
    \midrule
     Base  & 50.45 & 64.7  & 54.38 & 56.51 \\
        +Small lr  & 82.1  & 80.38 & 79.57 & 80.68 \\
          +Kaiming & 81.99 & 54.23 & 50.54 & 62.25 \\
          +FTM & \textbf{82.93} & -     & -     & - \\
          +MST & 82.32 & 81.06 & \underline{80.28} & \underline{81.22} \\
           +SOP  & 82.33 & \underline{81.31} & 51.31 & 71.65 \\
          +CSP  & \underline{82.69} & \textbf{81.55} & \textbf{82.05} & \textbf{82.09} \\
    \bottomrule
\end{tabularx}%
    }
  \caption{Average test performance~(\%) across different tokenizers and guiding tasks. 
  Base model with $k$-mer refers to DNABERT with its implementation flaws corrected.
  }
  \label{tab:guiding_task}%
\end{table}

\subsection{Result Analysis}
\noindent\textbf{Q1.} \textbf{\textit{Are these two criteria useful, and how can we select stable and valid downstream datasets?} }

\noindent\textbf{A1.} 
As shown in Table~\ref{tab:stability}, the standard deviation of performance on the K27m3 dataset reaches 2.85\%, while on K36m3, it even reaches 3.57\% under different pretraining initializations, indicating significant instability in experimental reproducibility. In addition, Table~\ref{tab:validity} shows validity issues in some datasets, e.g., a simple CNN outperforms pretraining models by 7.9\% in the NTP dataset. These phenomena suggest that some downstream datasets fail to fairly and validly reflect the capability of pretraining models.

We filtered datasets based on two criteria: stability and validity. As shown in Figure~\ref{fig:criteria}(a), the logarithm of the standard deviation distribution follows a normal distribution, with a Shapiro-Wilk test~\cite{shaphiro1965analysis} p-value of 0.15. According to the stability criterion~\ref{cri:stability}, we selected datasets with their standard deviation less than 1.41\%. Secondly, according to Criterion~\ref{cri:validity} and Table~\ref{tab:validity}, we removed datasets on which end-to-end models perform best (including NTP, K27a, PT, SSAcc, SSAll, SSD). As shown in Figure~\ref{fig:criteria}, according to the scaling law, we retained tasks with a positive slope and coefficient of determination $R^2 > 0.4$~\cite{hastie2009elements}. Finally, we selected seven datasets that are both stable and valid: CA, CpG, HM, EC, NE, PA, and PNT.

These seven datasets not only exhibit good properties but also cover various DNA downstream tasks, e.g., CA for chromatin accessibility prediction, HM for histone modification prediction, etc. As shown in Table~\ref{tab:performance}, the average performance of DNABERT, DNABERT2, and HyenaDNA on the seven datasets is 83.0\%, 81.4\%, and 81.1\%, respectively. The complete evaluation results are provided in the supplementary materials.

\noindent\textbf{Q2.} \textbf{\textit{Why is it necessary to use guiding tasks during DNA pretraining, and how does it work?} }

\noindent\textbf{A2.} 
As shown in Table~\ref{tab:guiding_task}, the base model with $k$-mer performs poorly due to its failure to converge. After introducing the FTM and MST, both of which originate from neighbor-masking flaws, the model performance improves by 32.48\% and 31.87\%, respectively. These flaws surprisingly assist convergence.
\textbf{In fact, difficulty converging during pretraining is commonly observed in DNA models.} As shown in Table~\ref{tab:guiding_task}, the base models with Word and BPE tokenizers also exhibit poor performance. Introducing guiding tasks significantly improves model performance; for instance, after applying the MST, the Word and BPE variants achieve 16.36\% and 25.9\% improvements over the base model, respectively. Pretraining loss curves in Figure~\ref{fig:guiding_task_loss}(a) further confirm that the base model faces difficulty in convergence, but shows stable convergence behavior after introducing guiding tasks. Therefore, it is necessary to incorporate guiding tasks during DNA pretraining to facilitate convergence.

\begin{figure}[!t]
    \centering
    \includegraphics[width=1\linewidth]{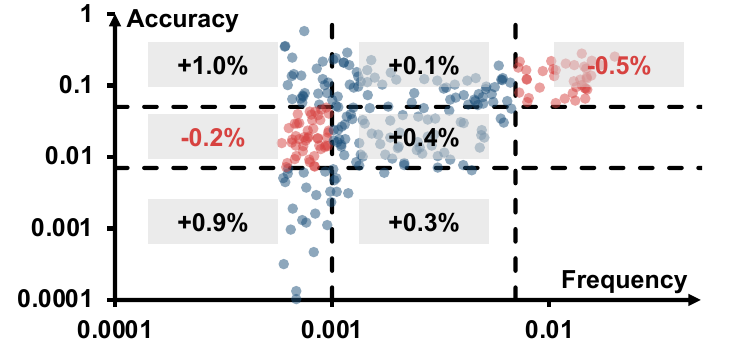}
    \caption{Visualization of token importance in the BPE vocabulary. The vocabulary size is 256, excluding special tokens. ``Accuracy'' refers to the token prediction accuracy of the model. The metric in each block indicates the improvement after removing the block's tokens compared to the base. The base's vocabulary is constructed by BPE, with the same size as the vocabulary after token removal.}
    \label{fig:vocab}
\end{figure}

Guiding tasks guide model initialization through simple tasks. 
Taking the CSP task as an example, as shown in Figures~\ref{fig:guiding_task_loss}, the model is initialized by learning CSP with an initially large loss. Then the CSP loss rapidly drops due to its simplicity. Then the well-initialized model can better understand the masked token prediction, thereby promoting convergence.
It is noteworthy that guiding-based initialization differs from traditional initialization methods (e.g., as shown in Figure~\ref{fig:guiding_task_loss}(a), Kaiming initialization fails to facilitate model convergence), demonstrating better robustness.

Furthermore, guiding tasks can significantly improve the model's performance. As shown in Table~\ref{tab:guiding_task}, although the Base$^*$ model with adjusted learning rate performs well, it fails to surpass guiding tasks variants. For instance, models with CSP outperform the Base$^*$ by 0.59\%, 1.17\%, and 2.48\% under $k$-mer, Word, and BPE, respectively.

\noindent\textbf{Q3.} \textbf{\textit{What are the advantages of Complementary Strand Prediction compared to other guiding tasks?} 
}

\noindent\textbf{A3.} 
{{(a)} The CSP task demonstrates stronger robustness}. As shown in Figure~\ref{fig:guiding_task_loss}(a), although the final pretraining loss of MST is close to that of CSP, the CSP loss decreases more smoothly during training. At the same time, the MST loss exhibits significant oscillations, indicating that CSP is more conducive to stabilizing the training process. 
{{(b)} The CSP task showed higher efficiency}. As shown in Table~\ref{tab:guiding_task}, under both Word and BPE tokenizers, the base model with CSP achieves the best performance. Compared with the second-best MST variants, CSP improves the average performance by 0.87\%.
{{(c)} The CSP task shows broader applicability}. Firstly, unlike FTM, which is only applicable to overlapping tokenizers, CSP is not restricted by tokenizers. Secondly, as shown in Table~\ref{tab:guiding_task}, CSP facilitates convergence across all variants, in contrast to SOP, which fails in BPE, and Kaiming initialization, which fails in Word and BPE.

These advantages stem from CSP design, which does not interfere with the optimization of MTP. Because it does not involve operations on the embedding of masked tokens, it avoids contamination of the token representation space. In contrast, MST introduces special tokens that do not appear in the natural distribution during training, which may mislead the model's learning of the real token distribution and thus interfere with the optimization direction of MTP, resulting in loss oscillations.

\begin{table}[!t]
  \centering
    \renewcommand{\arraystretch}{1.2}
    \scalebox{1}{
\begin{tabularx}{0.47\textwidth}{c|>{\centering\arraybackslash}X>{\centering\arraybackslash}X>{\centering\arraybackslash}X>{\centering\arraybackslash}X}
    \toprule
    Vocab Size & $k$-mer & Word  & BPE   & Average \\
    \midrule
    16    & \underline{83.09}  & \textbf{82.18 } & 79.57  & 81.61  \\
    64    & \underline{83.09}  & \underline{82.17}  & 82.78  & \textbf{82.68 } \\
    256   & 83.01  & 81.80  & \textbf{83.17 } & \underline{82.66}  \\
    1024  & \textbf{83.13 } & 81.70  & \underline{82.89}  & 82.57  \\
    4096  & 82.69  & 81.55  & 82.05  & 82.09  \\
    16384 & 82.62  & 80.48  & 82.06  & 81.72  \\
    \bottomrule
\end{tabularx}%
    }
  \caption{Average test performance~(\%) across different tokenizers and vocabulary sizes (excluding special tokens). 
  }
  \label{tab:vocab_size}%
\end{table}

\noindent\textbf{Q4.}\textbf{ \textit{What is the relationship between model performance and the DNA vocabulary?} }

\noindent\textbf{A4.} 
Not all tokens in a DNA vocabulary are necessary, especially in vocabularies constructed by BPE. Some tokens may even negatively impact performance. As shown in Figure~\ref{fig:vocab}, removing low-frequency tokens with high accuracy leads to a 1\% improvement in model performance; removing mid-frequency tokens also brings varying degrees of improvement. This suggests that some tokens may introduce redundant information or training noise. In contrast, high-frequency tokens with high accuracy usually carry more information and are essential to model performance; their removal causes a performance drop of about 0.5\%. We derive Principle~\ref{pri:token} based on these facts.

\begin{table}[!t]
  \centering
    \renewcommand{\arraystretch}{1.2}
    \scalebox{1}{
\begin{tabularx}{0.47\textwidth}{>{\centering\arraybackslash}X |c|>{\centering\arraybackslash}X>{\centering\arraybackslash}X>{\centering\arraybackslash}X>{\centering\arraybackslash}X}
    \toprule
    \multicolumn{1}{c|}{Freq.} & Acc. & $k$-mer  & Word  & BPE   & Avg. \\
    \midrule
    \multicolumn{2}{c|}{Base} & -     & -     & 82.38 & - \\
    \midrule
    \multirow{2}[2]{*}{High} & High  & 82.83 & 81.42 & 81.81 & 82.02 \\
          & Low   & \textbf{83.41} & 81.55 & -     & - \\
    \midrule
    \multirow{2}[2]{*}{Midium} & High  & 82.96 & 81.54 & 82.5  & 82.33 \\
          & Low   & 83.25 & \underline{81.59 }& 82.68 & 82.51 \\
    \midrule
    \multirow{2}[2]{*}{Low} & High  & 82.98 & \textbf{81.61} & \textbf{83.36} & \underline{82.65} \\
          & Low   & \underline{83.36 }& 81.52 & \underline{83.23} & \textbf{82.7} \\
    \bottomrule
\end{tabularx}%
    }
  \caption{Average test performance~(\%) across different tokenizers and token removal settings. 
  }
  \label{tab:vocab_construction}%
\end{table}

There is no linear correlation between model performance and vocabulary size. As shown in Table~\ref{tab:vocab_size}, the best-performing vocabulary sizes vary significantly across different tokenizers: the optimal sizes for $k$-mer, Word, and BPE tokenizers are 1024, 16, and 256, respectively. This indicates that the best vocabulary size depends on the tokenizer and the structural characteristics of DNA sequences.

The impact of the tokenization on token importance should not be overlooked. As shown in Table~\ref{tab:vocab_construction}, $k$-mer tokenizer introduces token overlapping, meaning that information from removed tokens can still be indirectly represented by others, resulting in minimal performance changes. In addition, the Word tokenizer directly segments the DNA sequence, severely disrupting its original structural integrity, which leads to worse performance compared to $k$-mer and BPE. Furthermore, the removal of tokens has a relatively limited impact. We derive Principle~\ref{pri:tokenization} based on these facts. Further experimental analysis on vocabulary size, etc., can be found in the supplementary material.

\section{Conclusion}


In this work, we systematically identify three overlooked issues in DNA sequence pretraining, including inappropriate downstream datasets, implementation flaws in the neighbor-masking strategy, and the lack of discussion for vocabulary construction, which fundamentally undermine the validity and effectiveness of existing methods. Through in-depth analysis and extensive empirical validation, we propose two objective criteria for downstream datasets selection, the novel concept of guiding tasks, and principles for constructing meaningful DNA vocabularies, to mitigate these issues. We finally introduce a standardized testbed that enables reproducible, rigorous benchmarking, which demonstrates the rationale of our findings and the effectiveness of our recommendations. Future work will focus on extending and promoting the testbed as a community-wide benchmarking standard to advance DNA pretraining.

\section{Acknowledgments}
This work was supported by the Strategic Priority Research Program of the Chinese Academy of Sciences under Grant No. XDA0460205.

{
    \small
    \bibliographystyle{ieeenat_fullname}
    \bibliography{main}

@article{encode2012integrated,
  title={An integrated encyclopedia of DNA elements in the human genome},
  author={ENCODE Project Consortium and others},
  journal={Nature},
  volume={489},
  number={7414},
  pages={57},
  year={2012}
}

@article{zhou2015predicting,
  title={Predicting effects of noncoding variants with deep learning--based sequence model},
  author={Zhou, Jian and Troyanskaya, Olga G},
  journal={Nature methods},
  volume={12},
  number={10},
  pages={931--934},
  year={2015},
  publisher={Nature Publishing Group US New York}
}

@article{hyenadna,
  title={Hyenadna: Long-range genomic sequence modeling at single nucleotide resolution},
  author={Nguyen, Eric and Poli, Michael and Faizi, Marjan and Thomas, Armin and Wornow, Michael and Birch-Sykes, Callum and Massaroli, Stefano and Patel, Aman and Rabideau, Clayton and Bengio, Yoshua and others},
  journal={Advances in neural information processing systems},
  volume={36},
  pages={43177--43201},
  year={2023}
}

@article{dnabert,
  title={DNABERT: pre-trained Bidirectional Encoder Representations from Transformers model for DNA-language in genome},
  author={Ji, Yanrong and Zhou, Zhihan and Liu, Han and Davuluri, Ramana V},
  journal={Bioinformatics},
  volume={37},
  number={15},
  pages={2112--2120},
  year={2021},
  publisher={Oxford University Press}
}

@article{NucleotideTransformer,
  title={Nucleotide Transformer: building and evaluating robust foundation models for human genomics},
  author={Dalla-Torre, Hugo and Gonzalez, Liam and Mendoza-Revilla, Javier and Lopez Carranza, Nicolas and Grzywaczewski, Adam Henryk and Oteri, Francesco and Dallago, Christian and Trop, Evan and de Almeida, Bernardo P and Sirelkhatim, Hassan and others},
  journal={Nature Methods},
  volume={22},
  number={2},
  pages={287--297},
  year={2025},
  publisher={Nature Publishing Group US New York}
}

@article{dnabert2,
  title={Dnabert-2: Efficient foundation model and benchmark for multi-species genome},
  author={Zhou, Zhihan and Ji, Yanrong and Li, Weijian and Dutta, Pratik and Davuluri, Ramana and Liu, Han},
  journal={arXiv preprint arXiv:2306.15006},
  year={2023}
}

@article{evo2,
  title={Genome modeling and design across all domains of life with Evo 2},
  author={Brixi, Garyk and Durrant, Matthew G and Ku, Jerome and Poli, Michael and Brockman, Greg and Chang, Daniel and Gonzalez, Gabriel A and King, Samuel H and Li, David B and Merchant, Aditi T and others},
  journal={BioRxiv},
  pages={2025--02},
  year={2025},
  publisher={Cold Spring Harbor Laboratory}
}

@article{MxDNA,
  title={Model decides how to tokenize: Adaptive dna sequence tokenization with mxdna},
  author={Qiao, Lifeng and Ye, Peng and Ren, Yuchen and Bai, Weiqiang and Liang, Chaoqi and Ma, Xinzhu and Dong, Nanqing and Ouyang, Wanli},
  journal={Advances in Neural Information Processing Systems},
  volume={37},
  pages={66080--66107},
  year={2024}
}

@article{bend,
  title={Bend: Benchmarking dna language models on biologically meaningful tasks},
  author={Marin, Frederikke Isa and Teufel, Felix and Horlacher, Marc and Madsen, Dennis and Pultz, Dennis and Winther, Ole and Boomsma, Wouter},
  journal={arXiv preprint arXiv:2311.12570},
  year={2023}
}

@article{Genomicbenchmarks,
  title={Genomic benchmarks: a collection of datasets for genomic sequence classification},
  author={Gre{\v{s}}ov{\'a}, Katar{\'\i}na and Martinek, Vlastimil and {\v{C}}ech{\'a}k, David and {\v{S}}ime{\v{c}}ek, Petr and Alexiou, Panagiotis},
  journal={BMC Genomic Data},
  volume={24},
  number={1},
  pages={25},
  year={2023},
  publisher={Springer}
}

@article{sievers2017k,
  title={K-mer content, correlation, and position analysis of genome DNA sequences for the identification of function and evolutionary features},
  author={Sievers, Aaron and Bosiek, Katharina and Bisch, Marc and Dreessen, Chris and Riedel, Jascha and Fro{\ss}, Patrick and Hausmann, Michael and Hildenbrand, Georg},
  journal={Genes},
  volume={8},
  number={4},
  pages={122},
  year={2017},
  publisher={MDPI}
}

@article{bpe,
  title={Neural machine translation of rare words with subword units},
  author={Sennrich, Rico and Haddow, Barry and Birch, Alexandra},
  journal={arXiv preprint arXiv:1508.07909},
  year={2015}
}

@article{zhang2023dnagpt,
  title={DNAGPT: A generalized pre-trained tool for versatile DNA sequence analysis tasks},
  author={Zhang, Daoan and Zhang, Weitong and Zhao, Yu and Zhang, Jianguo and He, Bing and Qin, Chenchen and Yao, Jianhua},
  journal={arXiv preprint arXiv:2307.05628},
  year={2023}
}

@article{FANTOM5,
  title={An atlas of active enhancers across human cell types and tissues},
  author={Andersson, Robin and Gebhard, Claudia and Miguel-Escalada, Irene and Hoof, Ilka and Bornholdt, Jette and Boyd, Mette and Chen, Yun and Zhao, Xiaobei and Schmidl, Christian and Suzuki, Takahiro and others},
  journal={Nature},
  volume={507},
  number={7493},
  pages={455--461},
  year={2014},
  publisher={Nature Publishing Group UK London}
}

@article{GENCODE,
  title={GENCODE 2021},
  author={Frankish, Adam and Diekhans, Mark and Jungreis, Irwin and Lagarde, Julien and Loveland, Jane E and Mudge, Jonathan M and Sisu, Cristina and Wright, James C and Armstrong, Joel and Barnes, If and others},
  journal={Nucleic acids research},
  volume={49},
  number={D1},
  pages={D916--D923},
  year={2021},
  publisher={Oxford University Press}
}

@article{JASPAR,
  title={JASPAR 2020: update of the open-access database of transcription factor binding profiles},
  author={Fornes, Oriol and Castro-Mondragon, Jaime A and Khan, Aziz and Van der Lee, Robin and Zhang, Xi and Richmond, Phillip A and Modi, Bhavi P and Correard, Solenne and Gheorghe, Marius and Barana{\v{s}}i{\'c}, Damir and others},
  journal={Nucleic acids research},
  volume={48},
  number={D1},
  pages={D87--D92},
  year={2020},
  publisher={Oxford University Press}
}

@article{ChIPAtlas,
  title={Ch IP-Atlas: a data-mining suite powered by full integration of public Ch IP-seq data},
  author={Oki, Shinya and Ohta, Tazro and Shioi, Go and Hatanaka, Hideki and Ogasawara, Osamu and Okuda, Yoshihiro and Kawaji, Hideya and Nakaki, Ryo and Sese, Jun and Meno, Chikara},
  journal={EMBO reports},
  volume={19},
  number={12},
  pages={e46255},
  year={2018}
}

@article{GENERator,
  title={GENERator: a long-context generative genomic foundation model},
  author={Wu, Wei and Li, Qiuyi and Li, Mingyang and Fu, Kun and Feng, Fuli and Ye, Jieping and Xiong, Hui and Wang, Zheng},
  journal={arXiv preprint arXiv:2502.07272},
  year={2025}
}

@article{BarcodeBERT,
  title={BarcodeBERT: Transformers for biodiversity analysis},
  author={Arias, Pablo Millan and Sadjadi, Niousha and Safari, Monireh and Gong, ZeMing and Wang, Austin T and Lowe, Scott C and Haurum, Joakim Bruslund and Zarubiieva, Iuliia and Steinke, Dirk and Kari, Lila and others},
  journal={arXiv preprint arXiv:2311.02401},
  year={2023}
}

@article{evo,
  title={Sequence modeling and design from molecular to genome scale with Evo},
  author={Nguyen, Eric and Poli, Michael and Durrant, Matthew G and Kang, Brian and Katrekar, Dhruva and Li, David B and Bartie, Liam J and Thomas, Armin W and King, Samuel H and Brixi, Garyk and others},
  journal={Science},
  volume={386},
  number={6723},
  pages={eado9336},
  year={2024},
  publisher={American Association for the Advancement of Science}
}

@article{su2024roformer,
  title={Roformer: Enhanced transformer with rotary position embedding},
  author={Su, Jianlin and Ahmed, Murtadha and Lu, Yu and Pan, Shengfeng and Bo, Wen and Liu, Yunfeng},
  journal={Neurocomputing},
  volume={568},
  pages={127063},
  year={2024},
  publisher={Elsevier}
}

@inproceedings{devlin2019bert,
  title={Bert: Pre-training of deep bidirectional transformers for language understanding},
  author={Devlin, Jacob and Chang, Ming-Wei and Lee, Kenton and Toutanova, Kristina},
  booktitle={Proceedings of the 2019 conference of the North American chapter of the association for computational linguistics: human language technologies, volume 1 (long and short papers)},
  pages={4171--4186},
  year={2019}
}

@article{ku2025systems,
  title={Systems and algorithms for convolutional multi-hybrid language models at scale},
  author={Ku, Jerome and Nguyen, Eric and Romero, David W and Brixi, Garyk and Yang, Brandon and Vorontsov, Anton and Taghibakhshi, Ali and Lu, Amy X and Burke, Dave P and Brockman, Greg and others},
  journal={arXiv preprint arXiv:2503.01868},
  year={2025}
}

@inproceedings{Unet,
  title={U-net: Convolutional networks for biomedical image segmentation},
  author={Ronneberger, Olaf and Fischer, Philipp and Brox, Thomas},
  booktitle={International Conference on Medical image computing and computer-assisted intervention},
  pages={234--241},
  year={2015},
  organization={Springer}
}

@book{MLP,
  title={Neural networks: a comprehensive foundation},
  author={Haykin, Simon},
  year={1998},
  publisher={Prentice Hall PTR}
}

@misc{GRCh38p14,
  author       = {{Genome Reference Consortium}},
  title        = {Human genome assembly GRCh38.p14 (GCF\_000001405.40)},
  year         = {2022},
  publisher    = {National Center for Biotechnology Information (NCBI)},
  howpublished = {\url{https://www.ncbi.nlm.nih.gov/assembly/GCF_000001405.40/}},
  note         = {Accessed: 2025-07-29}
}

@article{kingma2014adam,
  title={Adam: A method for stochastic optimization},
  author={Kingma, Diederik P and Ba, Jimmy},
  journal={arXiv preprint arXiv:1412.6980},
  year={2014}
}

@article{shaphiro1965analysis,
  title={An analysis of variance test for normality},
  author={Shaphiro, S and Wilk, MBJB},
  journal={Biometrika},
  volume={52},
  number={3},
  pages={591--611},
  year={1965}
}

@book{manning1999foundations,
  title={Foundations of statistical natural language processing},
  author={Manning, Christopher and Schutze, Hinrich},
  year={1999},
  publisher={MIT press}
}

@article{gardiner1987cpg,
  title={CpG islands in vertebrate genomes},
  author={Gardiner-Garden, Margaret and Frommer, Marianne},
  journal={Journal of molecular biology},
  volume={196},
  number={2},
  pages={261--282},
  year={1987},
  publisher={Elsevier}
}

@article{kouzarides2007chromatin,
  title={Chromatin modifications and their function},
  author={Kouzarides, Tony},
  journal={Cell},
  volume={128},
  number={4},
  pages={693--705},
  year={2007},
  publisher={Elsevier}
}

@article{radford2018improving,
  title={Improving language understanding by generative pre-training},
  author={Radford, Alec and Narasimhan, Karthik and Salimans, Tim and Sutskever, Ilya and others},
  year={2018},
  publisher={San Francisco, CA, USA}
}

@article{fawcett2006introduction,
  title={An introduction to ROC analysis},
  author={Fawcett, Tom},
  journal={Pattern recognition letters},
  volume={27},
  number={8},
  pages={861--874},
  year={2006},
  publisher={Elsevier}
}

@article{gorodkin2004comparing,
  title={Comparing two K-category assignments by a K-category correlation coefficient},
  author={Gorodkin, Jan},
  journal={Computational biology and chemistry},
  volume={28},
  number={5-6},
  pages={367--374},
  year={2004},
  publisher={Elsevier}
}

@book{hastie2009elements,
  title={The elements of statistical learning: data mining, inference, and prediction},
  author={Hastie, Trevor and Tibshirani, Robert and Friedman, Jerome H and Friedman, Jerome H},
  volume={2},
  year={2009},
  publisher={Springer}
}

@inproceedings{he2015delving,
  title={Delving deep into rectifiers: Surpassing human-level performance on imagenet classification},
  author={He, Kaiming and Zhang, Xiangyu and Ren, Shaoqing and Sun, Jian},
  booktitle={Proceedings of the IEEE international conference on computer vision},
  pages={1026--1034},
  year={2015}
}

@article{shannon1948mathematical,
  title={A mathematical theory of communication},
  author={Shannon, Claude E},
  journal={The Bell system technical journal},
  volume={27},
  number={3},
  pages={379--423},
  year={1948},
  publisher={Nokia Bell Labs}
}

@article{tishby2000information,
  title={The information bottleneck method},
  author={Tishby, Naftali and Pereira, Fernando C and Bialek, William},
  journal={arXiv preprint physics/0004057},
  year={2000}
}

@article{loshchilov2017decoupled,
  title={Decoupled weight decay regularization},
  author={Loshchilov, Ilya and Hutter, Frank},
  journal={arXiv preprint arXiv:1711.05101},
  year={2017}
}
}

%

\newpage

The supplementary material provides additional details to complement the main paper. 
\textbf{More details for method}, we further prove the information leakage issue of the neighbor-masking strategy and provide quantitative results. We also supplement the definitions of evaluation metrics, etc.  
\textbf{More details for datasets}, we offer more details about the characteristics and scales of downstream datasets.  
\textbf{More details for implementation details}, we supplement the parameter settings of each model.  
\textbf{More details for results and analysis}, we provide the complete data of the tables and figures in the main paper, along with additional analyses.

\section*{Method}
\subsection*{Two Implementation Flaws and Guiding Tasks}
$k$-mer tokenizer is widely used in DNA pretrain models. DNABERT mitigates the information leakage caused by overlapping tokenization by additionally masking the neighbors of masked tokens. However, this masking mechanism has two flaws. 
(a) Information leakage still exists. We present the following theorem:
\begin{theorem}
\label{the:difficulty}
Consider a sequence tokenized into overlapping $k$-mers, and suppose \( m \) consecutive tokens are masked, the leakage ratio $r$ is:
$$
    r=\begin{cases}
    100\%,&m \leq k - 1,\\
    \frac{k - 1}{m}\times 100\%, &   m > k - 1.
\end{cases}
$$
\end{theorem}
\noindent\textbf{Proof.} Suppose there exists a token sequence $\bm{x} = (t_1, t_2, \dots, t_m)$. We replace a continuous span of $m$ tokens in $\bm{x}$ with the [MASK] token, resulting in a masked sequence $\bm{\tilde{x}} = (t_1, \text{[MASK]}, \ldots, t_m)$, where the set of tokens to be predicted is denoted as $\mathcal{M}$. The conditional entropy of the model's prediction for $\mathcal{M}$ is given by:
\begin{gather}
H(\mathcal{M}|\bm{\tilde{x}}) = -\sum_{t_i \in \mathcal{M}}\sum_{t \in \mathcal{V}} p \log p,
\end{gather}
where $p = P(t_i = t | \bm{\tilde{x}})$ and $\mathcal{V}$ denotes the vocabulary. According to the Maximum Entropy Theorem~\cite{shannon1948mathematical}, the maximum value of the entropy above is $m\log|\mathcal{V}|$. Entropy reflects the uncertainty of the prediction target, i.e., the "intrinsic complexity" of the learning task~\cite{tishby2000information}.

Next, we consider the case where the input token sequence $\bm{x}$ is derived from the original DNA sequence using a $k$-mer tokenizer. To formally analyze the potential information leakage introduced by the $k$-mer tokenizer, we first illustrate with a concrete example:

Let token $t_i = (c_{i,1}, c_{i,2}, \dots, c_{i,k})$ and its preceding token $t_{i-1} = (c_{i-1,1}, c_{i-1,2}, \dots, c_{i-1,k})$. Due to the sliding-window nature of $k$-mer tokenization, we have $c_{i,j} = c_{i-1,j+1}$ for all $1 \leq j \leq k-1$. Thus, $t_{i-1}$ can be rewritten as $(c_{i-1,1}, c_{i,1}, \dots, c_{i,k-1})$, indicating that $t_i$ and $t_{i-1}$ differ by only one nucleotide, $c_{i-1,1}$.
Under this condition, if $t_{i-1}$ is not masked, then when predicting $t_i$, the model's search space is reduced from the original $|\mathcal{V}| = 4^k$ to a size of 4, as only the last nucleotide of $t_i$ remains uncertain. In other words, the $k$-mer structure significantly reduces the prediction difficulty when adjacent tokens are not masked.

When masking $m$ consecutive tokens, the entropy can be approximately expressed as:
\begin{align}
H'(\mathcal{M}|\bm{\tilde{x}}) &=
-\sum_{t_i\in \mathcal{M}}\sum_{t\in V^{(i)}}p\log p, \\
&\le -\sum_{t_i\in \mathcal{M}}\log|V^{(i)}|, \\
|V^{(i)}| &= 4^{k - \min\left(k,\max(0,k-i)+\max(0,k-m+i-1)\right)},
\end{align}
where $V^{(i)}$ is the effective candidate space for predicting $t_i$ given $\bm{\tilde{x}}$. Its size satisfies $4 \leq |V^{(i)}| \leq 4^k$, which can be directly derived from the structure of $k$-mers.

Therefore, we can derive the ratio between the maximum entropy under information leakage and the original maximum entropy as:
\begin{gather}
\frac{\max(H')}{\max(H)} =
\begin{cases}
0, & m \leq k - 1, \\
\frac{m - k + 1}{m}, & m > k - 1.
\end{cases}
\end{gather}

Hence, the proportion of information leakage can be expressed as $1 - \frac{\max(H')}{\max(H)}$, which quantifies the reduction in prediction difficulty introduced by $k$-mer tokenization, and reveals its potential to cause leakage in pretraining.

The above theorem indicates that information leakage exists regardless of the length of the consecutive masked token sequence. Such leakage weakens the pretraining model's ability to learn the semantics of DNA and instead encourages it to exploit the $k$-mer tokenization. In DNABERT~\cite{dnabert}, consecutively masking $k$ tokens leads to an information leakage rate of $r = \frac{k - 1}{k} \times 100\%$, which essentially reduces the task to masked nucleotide prediction.


\subsection*{DNA Sequence Segmentation}
In real-world DNA sequences obtained from biological experiments, the character ``N'' is often used to denote an unknown nucleotide. Traditional tokenization methods, such as $k$-mer, typically treat sequences containing ``N'' as out-of-vocabulary and map them to a special [UNK] token.

To systematically investigate the impact of ``N'' nucleotides, we conduct ablation studies on different strategies for handling ``N''. A straightforward approach is to remove all ``N'' characters from the sequences before tokenization. Another intuitive and biologically reasonable strategy is to define dedicated tokens for subsequences containing ``N''. To support this, we redesign the $k$-mer and Word tokenization by incorporating a greedy matching mechanism that identifies and preserves substrings involving ``N''. The modified tokenization procedure is described in Algorithm~\ref{alg:segment_dna}, where the input $T$ denotes a priority queue reverse ordered by the length of ``N''-containing token candidates.

\subsection*{Vocabulary Construction}
We experimentally found that not all tokens in mainstream vocabularies contribute to the pretraining model. After removing certain tokens, the model's capability even improves. Therefore, we define the vocabulary after pruning certain tokens as follows:
\begin{gather}
    \mathcal{V}_{\text{Cull}}=\mathcal{V}\setminus \mathcal{V}_{0}\cup\{[\text{CULL}]\}
\end{gather}
where \(\mathcal{V}\) is the vocabulary from which tokens are to be pruned, \(\mathcal{V}_{0}\) denotes the tokens to be removed, and \([\text{CULL}]\) is a newly added token that replaces the pruned tokens in the token sequence.

\begin{algorithm}[!ht]
\caption{Segment DNA Sequence with ``N''}
\label{alg:segment_dna}
\textbf{Input}: DNA sequence $\textit{seq}$, $k$-mer, ``N'' token priority list $\textit{T}$\\
\textbf{Output}: Token list $\textit{result}$\\
\vspace{-8pt}
\begin{algorithmic}[1]
\STATE Initialize empty list $\textit{result}$, set $i \gets 0$, $n \gets \text{Length}(\textit{seq})$
\WHILE{$i < n$}
    \STATE $j \gets i$
    \IF{$\textit{seq}[i] = \text{`N'}$}
        \WHILE{$j < n$ and $\textit{seq}[j] = \text{`N'}$}
            \STATE $j \gets j + 1$
        \ENDWHILE
        \STATE $\textit{len} \gets j - i$
        \FOR{each token $t$ in $T$}
            \STATE $m \gets \left\lfloor \frac{\textit{len}}{\text{Length}(t)} \right\rfloor$
            \IF{$m > 0$}
                \FOR{$1$ to $m$} \STATE Append $t$ to $\textit{result}$ \ENDFOR
                \STATE $i \gets i + m \cdot \text{Length}(t)$
                \STATE \textbf{break}
            \ENDIF
        \ENDFOR
    \ELSE
        \WHILE{$j < n$ and $\textit{seq}[j] \ne  \text{`N'}$}
            \STATE $j \gets j + 1$
        \ENDWHILE
        \FOR{$k' = i$ to $j - k$ step $k$}
            \STATE Append $\textit{seq}[k':k'+k]$ to $\textit{result}$
        \ENDFOR
        \STATE $i \gets j$
    \ENDIF
\ENDWHILE
\STATE \textbf{return} $\textit{result}$
\end{algorithmic}
\end{algorithm}

\begin{table*}[!ht]
  \centering
      \renewcommand{\arraystretch}{1.2}
          \scalebox{0.83}{
\begin{tabularx}{1.2\textwidth}{>{\centering\arraybackslash}X|l|>{\centering\arraybackslash}X>{\centering\arraybackslash}X>{\centering\arraybackslash}X>{\centering\arraybackslash}X>{\centering\arraybackslash}X}
    \toprule
    Benchmark & \multicolumn{1}{c|}{Dataset} & Aliase & \# Samples & Length range & Metric & \# Classes \\
    \midrule
    \multirow{3}[2]{*}{BEND} & Chromatin Accessibility & CA    & 2005617 & 512   & AUROC & 125 \\
          & Cpg Methylation & CpG   & 959039 & 512   & AUROC & 7 \\
          & Histone Modification & HM    & 612081 & 512   & AUROC & 18 \\
    \midrule
    \multirow{5}[2]{*}{Genomic} & Enhancers Cohn & EC    & 27791 & 500   & Accuracy & 2 \\
          & Enhancers Ensembl & EE    & 154842 & 269$\pm 122.6$ & Accuracy & 2 \\
          & Ensembl Regulatory & ER    & 289061 & 401$\pm 184.3$ & Accuracy & 3 \\
          & NonTata Promoters & NTP   & 36131 & 251   & Accuracy & 2 \\
          & OCR Ensembl & OCR   & 174756 & 315$\pm 108.1$ & Accuracy & 2 \\
    \midrule
    \multicolumn{1}{c|}{\multirow{18}[2]{*}{Nucleotide}} & Enhancers & NE    & 33000 & 400   & AUROC & 2 \\
          & Enhancers Types & NET   & 33000 & 400   & Accuracy & 3 \\
          & H2AFZ & Z     & 33000 & 1000  & MCC   & 2 \\
          & H3K27ac & K27a  & 31616 & 1000  & MCC   & 2 \\
          & H3K27me3 & K27m3 & 33000 & 1000  & MCC   & 2 \\
          & H3K36me3 & K36m3 & 33000 & 1000  & MCC   & 2 \\
          & H3K4me1 & K4m1  & 33000 & 1000  & MCC   & 2 \\
          & H3K4me2 & K4m2  & 32138 & 1000  & MCC   & 2 \\
          & H3K4me3 & K4m3  & 30776 & 1000  & MCC   & 2 \\
          & H3K9ac & K9a   & 24278 & 1000  & MCC   & 2 \\
          & H3K9me3 & K9m3  & 28288 & 1000  & MCC   & 2 \\
          & H4K20me1 & K20m1 & 32270 & 1000  & MCC   & 2 \\
          & Promoter All & PA    & 31584 & 300   & AUROC & 2 \\
          & Promoter NoTata & PNT   & 31372 & 300   & AUROC & 2 \\
          & Promoter Tata & PT    & 5274  & 300   & AUROC & 2 \\
          & Splice Sites Acceptor & SSAcc & 33000 & 600   & AUROC & 2 \\
          & Splice Sites All & SSAll & 33000 & 600   & Accuracy & 3 \\
          & Splic Sites Donor & SSD   & 33000 & 600   & AUROC & 2 \\
    \bottomrule
    \end{tabularx}%
    }
    \caption{Summary of scale information for each downstream dataset. ``BEND'', ``Genomic'', and ``Nucleotide'' refer to the BEND benchmark~\cite{bend}, Genomic benchmark~\cite{Genomicbenchmarks}, and Nucleotide Transformer benchmark~\cite{NucleotideTransformer}, respectively.}
  \label{tab:datasets}%
\end{table*}%

\subsection*{Evaluation Metrics}
To comprehensively assess model performance, we adopt three widely used evaluation metrics: Area Under the Receiver Operating Characteristic Curve (AUROC), Accuracy, and Matthews Correlation Coefficient (MCC).

\paragraph{AUROC.} AUROC measures the area under the ROC curve, which plots the true positive rate (TPR) against the false positive rate (FPR) at various threshold settings. It evaluates the model's ability to distinguish between positive and negative samples, and is especially robust in scenarios with class imbalance. An AUROC of 1.0 indicates perfect separation, while 0.5 corresponds to random guessing.

\paragraph{Accuracy.} Accuracy is the ratio of correctly predicted instances to the total number of instances. 

\paragraph{Matthews Correlation Coefficient (MCC).} MCC is a balanced measure that accounts for all four elements of the confusion matrix. It is particularly useful in evaluating binary classification performance on imbalanced data. MCC is defined as:
\begin{gather}
\scalebox{1}{$\text{MCC} = \frac{TP \times TN - FP \times FN}
{\sqrt{(TP + FP)(TP + FN)(TN + FP)(TN + FN)}},$}
\end{gather}
where $TP$, $TN$, $FP$, and $FN$ denote the number of true positives, true negatives, false positives, and false negatives, respectively. MCC ranges from $-1$ (perfect inverse prediction) to $1$ (perfect prediction), with 0 indicating no better than random guessing.

\begin{table*}[!htbp]
  \centering
    \renewcommand{\arraystretch}{1.2}
    \scalebox{0.83}{
\begin{tabularx}{1.2\textwidth}{c|l|>{\centering\arraybackslash}X>{\centering\arraybackslash}X>{\centering\arraybackslash}X>{\centering\arraybackslash}X>{\centering\arraybackslash}X>{\centering\arraybackslash}X>{\centering\arraybackslash}X>{\centering\arraybackslash}X}
    \toprule
    \multicolumn{1}{c|}{Tokenizer} & Variants & CA    & CpG   & HM    & EC    & NE    & PA    & PNT   & Average \\
    \midrule
    \multirow{3}[0]{*}{$k$-mer} & Base  & 70.33 & 88.8  & 76.81 & 71.29 & 82.47 & 94.39 & 94.71 & 82.69 \\
       & w/o N & 67.65 & 85.75 & 75.8  & 70.68 & 80.92 & 93.7  & 94.85 & 81.34 \\
        & Seg N & 69.61 & 87.17 & 76.25 & 71.79 & 81.95 & 94.23 & 94.87 & 82.27 \\
    \midrule
    \multirow{3}[0]{*}{Word} & Base  & 68.81 & 86.61 & 75.4  & 70.45 & 81.11 & 94.02 & 94.42 & 81.55 \\
       & w/o N & 67.41 & 85.1  & 74.48 & 69.14 & 79.16 & 93.37 & 94.22 & 80.41 \\
        & Seg N & 67.81 & 86.38 & 75.23 & 69.73 & 80.35 & 93.58 & 94.63 & 81.1 \\
    \bottomrule
\end{tabularx}%
    }
  \caption{Ablation studies on three methods for handling ``N'' nucleotides. Test performance~(\%) (Accuracy, AUROC, or MCC) on selected datasets. ``Base'' denotes processing substrings containing ``N'' as [UNK] token. ``w/o N'' directly removes ``N'' nucleotides from the pretraining dataset. ``Seg N'' applies Algorithm~\ref{alg:segment_dna} to segment them into ``N'' tokens.
  }
  \label{tab:ablation_N}%
  \vspace{-4pt}
\end{table*}

\underline

\section*{Datasets}
The following introduces the downstream datasets that are widely used for evaluating DNA pretraining models. These datasets span various biological tasks and have been adopted in several recent studies to benchmark the effectiveness of sequence representations learned during pretraining.

\noindent\textbf{BEND Benchmark}~\cite{bend} introduces a collection of downstream tasks in the field of regulatory genomics, designed to evaluate the ability of models to capture sequence-level functional signals. We select three supervised binary classification tasks from this benchmark:
\begin{itemize}
    \item \textbf{Chromatin Accessibility Prediction (CA):} This task requires the model to predict whether a given DNA sequence is located within an open chromatin region across diverse cell types. Open chromatin regions are key indicators of regulatory activity. The dataset contains a total of 2,005,617 samples.
    
    \item \textbf{Histone Modification Prediction (HM):} A multi-label task in which the model must determine whether certain histone modifications (e.g., H3K27ac, H3K4me3) occur at specific genomic loci. These modifications play essential roles in gene regulation. The HM dataset includes 612,081 DNA sequence samples.
    
    \item \textbf{CpG Methylation Prediction (CpG):} This task involves predicting the methylation state of CpG sites in the genome, which is critical for epigenetic regulation. The dataset includes 959,039 samples collected from various cell lines and methylation profiles.
\end{itemize}
We use the AUROC as the evaluation metric, following the original benchmark setting.

\noindent\textbf{Genomic Benchmark}~\cite{Genomicbenchmarks} provides a diverse suite of classification datasets that can be grouped into two major tasks:
\begin{itemize}
    \item \textbf{Regulatory Annotation:} This task focuses on identifying regulatory DNA elements such as promoters and enhancers, which are central to transcriptional control. It includes four datasets: Enhancers Cohn~(EC), Enhancers Ensembl~(EE), Ensembl Regulatory~(ER), and NonTATA Promoters~(NTP).

    \item \textbf{Chromatin Accessibility Prediction:} Similar to the CA task from BEND, this task uses the OCR Ensembl~(OCR) dataset, which consists of 174,756 samples, to predict open chromatin regions from sequence input.
\end{itemize}
Following prior works~\cite{hyenadna, MxDNA}, accuracy is used as the evaluation metric.

\noindent\textbf{Nucleotide Transformer Benchmark}~\cite{NucleotideTransformer} evaluates the generalization ability of pretrained DNA language models across a broader set of tasks and datasets. It covers three main categories:
\begin{itemize}
    \item \textbf{Histone Modification Prediction:} This category includes 10 binary classification datasets targeting different histone marks: H2AFZ (Z), H3K27ac (K27a), H3K27me3 (K27m3), H3K36me3 (K36m3), H3K4me1 (K4m1), H3K4me2 (K4m2), H3K4me3 (K4m3), H3K9ac (K9a), H3K9me3 (K9m3), and H4K20me1 (K20m1). These datasets are evaluated using MCC, which is suitable for imbalanced classification.

    \item \textbf{Regulatory Annotation:} This includes five datasets focused on identifying regulatory elements: Enhancers (NE), Enhancer Types (NET), Promoter All (PA), Promoter NoTATA (PNT), and Promoter TATA (PT). AUROC is used as the evaluation metric for all datasets, except for NET, which uses accuracy.

    \item \textbf{Splice Site Annotation:} This task involves detecting splice junctions—the boundaries between exons and introns—based on local sequence context. The datasets include Splice Sites Acceptor (SSAcc), Splice Sites All (SSAll), and Splice Sites Donor (SSD). AUROC is used as the metric for SSAcc and SSD, while SSAll is evaluated using accuracy due to its multi-class setting.
\end{itemize}

\begin{figure*}[!ht]
    \centering
    \includegraphics[width=1\linewidth]{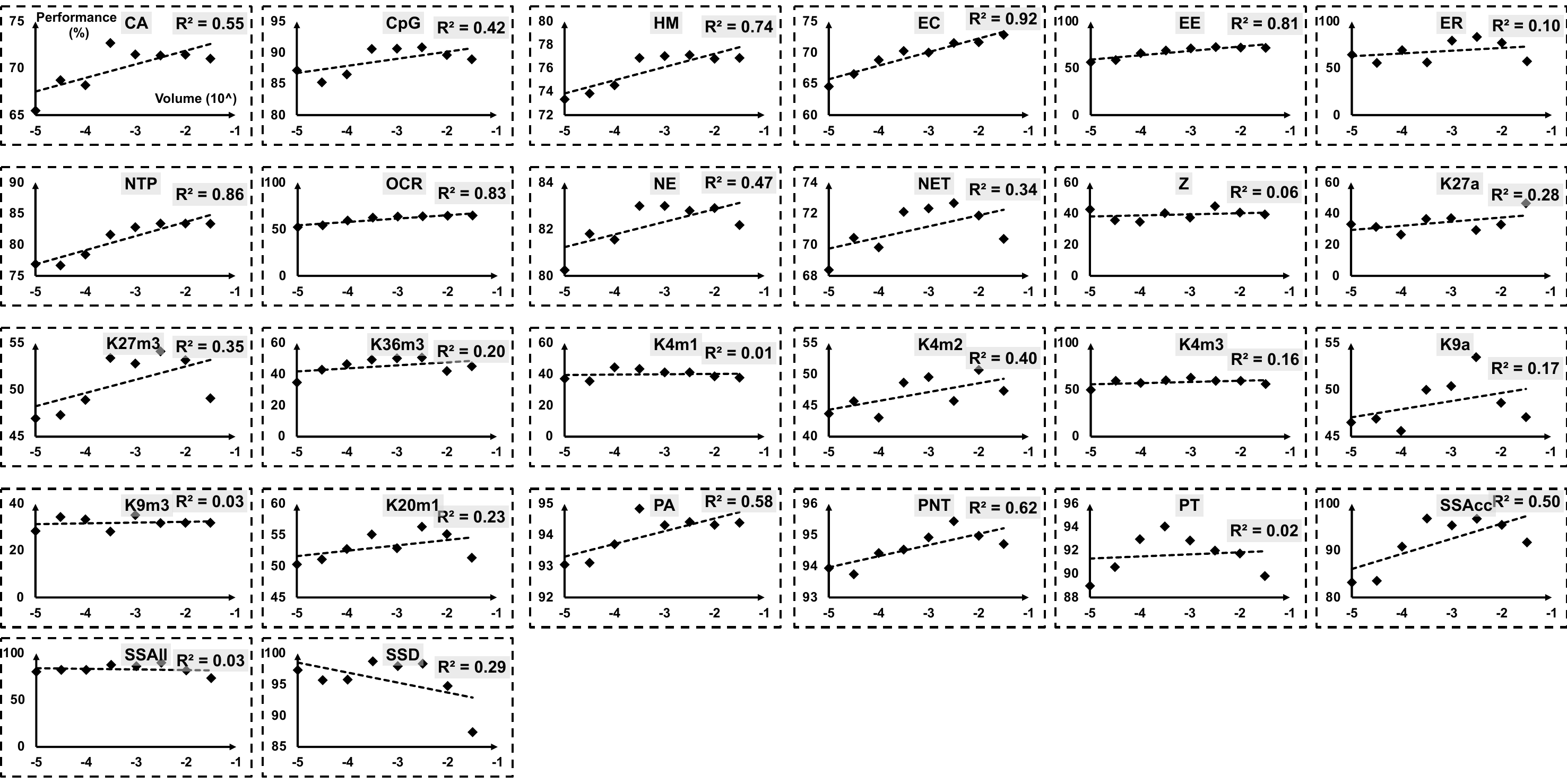}
    \caption{Supplementary for Figure 3 in the main paper. The scaling law adherence status of all 26 downstream datasets. We model the scaling law in Criteria 2 using linear regression~\cite{hastie2009elements}.}
    \label{fig:sup_scaling_law}
\end{figure*}
\begin{figure*}[!ht]
    \centering
    \includegraphics[width=1\linewidth]{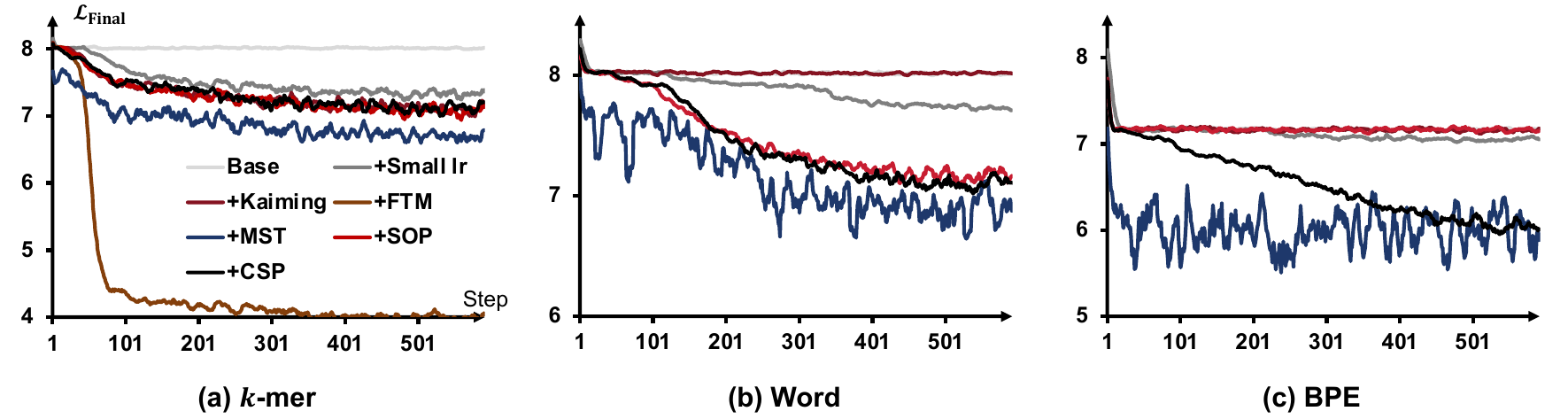}
    \caption{Supplementary for Figure 4 in the main paper. Visualization of pretraining loss curves for different guiding tasks and different tokenizers. ``Base'' indicates that DNABERT is pretrained at a learning rate of $4e-4$.}
    \label{fig:sup_guiding_task_loss}
\end{figure*}

\section*{Implementations}
\noindent\textbf{Criteria.} We obtain the estimates of the mean $\mu$ and variance $\sigma^2$ of $\log(\sigma_{\mathcal{D}})$ through extensive experiments. We evaluate the performance of three baselines: Multilayer Perceptrons (MLP)~\cite{MLP}, Convolutional Neural Networks (CNN)~\cite{bend}, and Unet~\cite{Unet}, and compare them with DNABERT to verify pretraining benefit. We model the scaling law using linear regression~\cite{hastie2009elements}.

\noindent\textbf{Guiding Task.}
We conduct experiments based on the fixed version with different guiding tasks, and include the Kaiming initialization~\cite{he2015delving} method for comparison. The probability of a reverse in SOP is set to 0.01.

\noindent\textbf{Vocabulary Construction.}
All BPE vocabularies are trained on the human genome data~\cite{GRCh38p14}. In the token ablation experiments, the removed tokens are replaced with the new $\text{[CUL]}$ token, noting that $\text{[CUL]}$ is not a special token. And the number of removed tokens does not exceed 10\% of the raw vocabulary.

\noindent\textbf{Models.}
Unless otherwise specified, all $k$-mer settings use $k=6$. DNABERT and DNABERT2 are configured with a maximum position embedding of 512, a hidden dimension of 256, and a total of eight Transformer encoder layers.  
HyenaDNA uses a maximum position embedding of 8192, a hidden dimension of 256, and eight hidden layers.  
The CNN model initializes sequences using the BPE and contains two CNN layers with a hidden dimension of 256. The MLP model consists of two hidden layers with a dimension of 1024. The U-Net model uses a hidden dimension of 32.

\noindent\textbf{Experiments.}
All pretrainings use human genome data and were conducted on NVIDIA A800 GPUs with a batch size of 64 over 100 epochs. We use the AdamW optimizer~\cite{loshchilov2017decoupled} for optimization, with a linear warmup schedule where the number of warmup epochs is set to 10. The learning rate is $4 \times 10^{-4}$. The masking probability is 0.11.
All test experiments were conducted on NVIDIA A800 GPUs with 100 epochs of three different seeds, using the Adam optimizer. Following prior work~\cite{bend}, we freeze the pretraining model's backbone and only fine-tune a CNN with a learning rate of $5 \times 10^{-4}$. Other models undergo full training with multiple learning rates $\{1\!\times\!10^{-3}, 5\!\times\!10^{-4}, 1\!\times\!10^{-4}\}$. Different datasets use different batch sizes. For samples whose length exceeds the model's maximum position embedding, we split them into segments, feed each segment into the model separately, and then concatenate the resulting outputs to obtain the final embeddings.

\begin{figure*}
    \centering
    \includegraphics[width=1\linewidth]{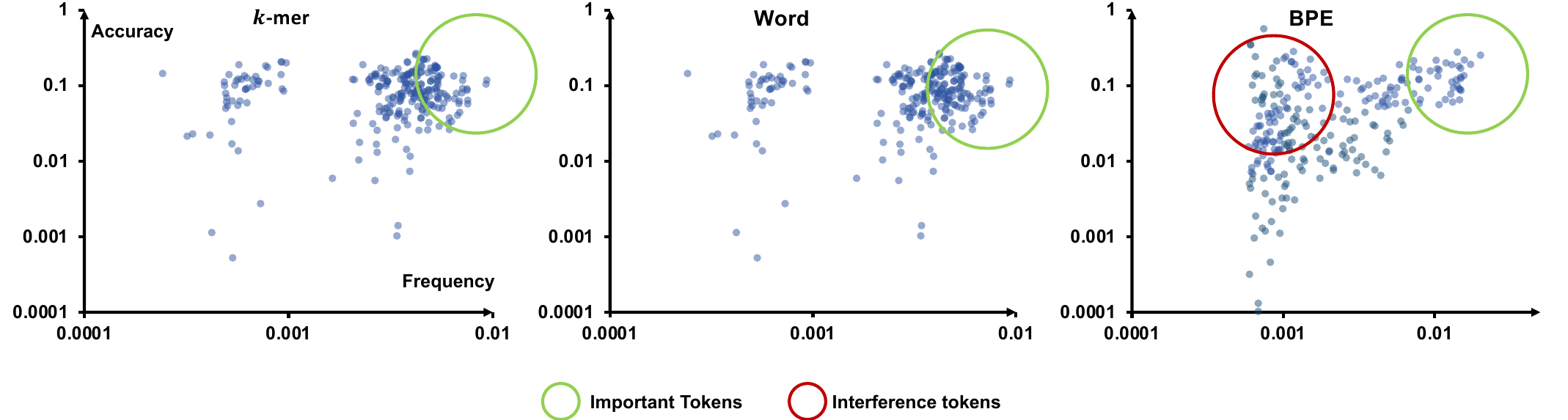}
    \caption{Visualization of token distribution. All vocabulary sizes are 256, excluding special tokens. “Accuracy” refers to the token prediction accuracy of the model.}
    \label{fig:sup_vocab}
\end{figure*}

\section*{Experiments}
\subsection*{Complete Results of The Two Criteria}
\noindent\textbf{Full Table of the Two Criteria.}
In Table~\ref{tab:stability} and \ref{tab:validity}, we provide the complete data corresponding to Table 1 and Table 2 in the main text. The full tables include stability and validity results for all 26 downstream datasets. Additionally, the scaling laws for all datasets are illustrated in Figure~\ref{fig:sup_scaling_law}, where it can be observed that datasets such as ER and PT do not satisfy our proposed Criteria 2. Finally, we selected seven datasets that are both stable and valid: CA, CpG, HM, EC, NE, PA, and PNT.

\noindent\textbf{SOTA Pretraining Model Performance.}
As shown in Table~\ref{tab:performance}, DNABERT achieves the best overall performance. DNABERT2, which adopts a BPE vocabulary, fails to reproduce the performance reported in its original paper, likely due to the incompatibility of the vocabulary. As shown in Table~\ref{tab:vocab_size}, BPE-based variants perform comparably to the $k$-mer tokenizer. HyenaDNA exhibits inferior performance on tasks such as CpG and EC compared to other SOTA models, resulting in the lowest average performance.

\subsection*{Guiding Tasks}
FTM demonstrates strong capability in facilitating pretraining convergence. As shown in Figure~\ref{fig:sup_guiding_task_loss} (a), and following Theorem~\ref{the:difficulty}, FTM significantly reduces the difficulty of the masked token prediction task, leading to a rapid decrease in loss. As shown in Table~\ref{tab:guiding_task}, among all $k$-mer variants, FTM achieves the best performance. However, due to its reliance on information leakage, FTM is only applicable to tokenizers with overlapping structures, such as $k$-mer.

In contrast, CSP and MST have broader applicability. As illustrated in Figure~\ref{fig:sup_guiding_task_loss}, both CSP and MST effectively promote convergence during pretraining across all three tokenizer variants: $k$-mer, Word, and BPE. This is because the objectives of CSP and MST are independent of the tokenizer design. Between the two, CSP exhibits more stable behavior with a smoother pretraining loss curve, whereas MST tends to fluctuate more severely.

\subsection*{Vocabulary Construction}
\noindent\textbf{Vocabulary Size.} There is no linear correlation between model performance and vocabulary size. As shown in Table~\ref{tab:vocab_size}, the best-performing vocabulary sizes vary significantly across different tokenizers: the optimal sizes for $k$-mer, Word, and BPE tokenizers are 1024, 16, and 256, respectively. This indicates that the best vocabulary size depends on the tokenizer and the structural characteristics of DNA sequences.

\noindent\textbf{``N'' Nucleotides.}
``N'' nucleotides are beneficial for model learning. As shown in Table~\ref{tab:ablation_N}, the ``w/o N'' variant performs the worst, indicating that these ``N'' nucleotides still contain semantic information in DNA. The ``Seg N'' variant performs 0.4\% worse than the base, suggesting that there remains considerable room for exploration regarding how to tokenize ``N'' nucleotides effectively.

\noindent\textbf{Token Importance.}
We supplement the token importance visualizations for $k$-mer and Word vocabularies in Figure~\ref{fig:sup_vocab}. Combined with the statistics in Table~\ref{tab:vocab_construction}, tokens with high frequency and high prediction accuracy are essential to the model, i.e., removing them leads to noticeable performance drops: 0.08\%, 0.38\%, and 1.36\% for $k$-mer, Word, and BPE, respectively, compared to the raw model. In contrast, low-frequency but high-accuracy tokens are often noisy, i.e., their inclusion may degrade model performance. For example, removing such tokens improves BPE variant performance by 0.98\% and 0.19\% compared to the base and raw models, respectively. Based on these observations, we propose Principle 3 in the main text.

As shown in Table~\ref{tab:vocab_construction}, the $k$-mer tokenizer introduces overlapping tokens, which allows information from removed tokens to be partially retained through neighboring tokens. This redundancy mitigates the impact of token removal and results in only minor performance degradation. On the other hand, the Word tokenizer segments DNA sequences in a non-overlapping manner, which substantially disrupts the inherent structural continuity of the sequence, leading to significantly worse performance compared to $k$-mer and BPE. Additionally, token removal has a relatively limited effect on performance across all tokenizers.


\end{document}